\documentclass[nohyperref]{article}

\usepackage[accepted]{icml2022}

\usepackage{hyperref}

\usepackage{algorithmic}

\usepackage{xspace} 
\usepackage{xspace} 
\usepackage{amsmath}
\usepackage{amssymb}
\usepackage{mathtools}
\usepackage{amsthm}

\usepackage[capitalize,noabbrev]{cleveref}

\usepackage[utf8]{inputenc} % allow utf-8 input
\usepackage[T1]{fontenc}    % use 8-bit T1 fonts
\usepackage{hyperref}       % hyperlinks
\usepackage{url}            % simple URL typesetting
\usepackage{booktabs}       % professional-quality tables
\usepackage{amsfonts}       % blackboard math symbols
\usepackage{nicefrac}       % compact symbols for 1/2, etc.
\usepackage{microtype}      % microtypography
\usepackage{amsmath, amssymb}
\usepackage[english]{babel}
\usepackage{thmtools,thm-restate}

\usepackage{algorithm}
\usepackage{algorithmic} 
\usepackage{xspace}
\usepackage{bm}
\usepackage{graphicx}
\usepackage{subfig}
\usepackage{comment}

\usepackage{xspace} 
\newcommand{\alg}{{\textsc{AdaCore}}\xspace}
\newcommand{\craig}{{\textsc{Craig}}\xspace}
\newcommand{\glister}{{\textsc{Glister}}\xspace}
\newcommand{\gradmatch}{{\textsc{GradMatch}}\xspace}

\newcommand{\overbar}[1]{\mkern 1.5mu\overline{\mkern-1.5mu#1\mkern-1.5mu}\mkern 1.5mu}

\usepackage[textsize=tiny]{todonotes}

\newcommand{\tsn}[1]{{\left\vert\kern-0.25ex\left\vert\kern-0.25ex\left\vert #1 
    \right\vert\kern-0.25ex\right\vert\kern-0.25ex\right\vert}}
\usepackage{xcolor}
\definecolor{darkred}{RGB}{150,0,0}
\definecolor{darkgreen}{RGB}{0,150,0}
\definecolor{darkblue}{RGB}{0,0,200}
\newtheorem{theorem}{Theorem}[section]

\newtheorem{corollary}[theorem]{Corollary}

\newcommand{\beq}{\begin{equation}}

\newcommand{\eeq}{\end{equation}}

\newcommand{\Lc}{{\cal{L}}}

\newcommand{\R}{\mathbb{R}}

\def \endprf{\hfill {\vrule height6pt width6pt depth0pt}\medskip}
\icmltitlerunning{Second-order Coresets for Training Machine Learning Models}

\begin{document}

\twocolumn[
\icmltitle{Adaptive Second Order Coresets for Data-efficient
Machine Learning}




\begin{icmlauthorlist}
\icmlauthor{Omead Pooladzandi}{yyy}
\icmlauthor{David Davini}{sch}
\icmlauthor{Baharan Mirzasoleiman}{sch}
\end{icmlauthorlist}

\icmlaffiliation{yyy}{Department of Electrical \& Computer Engineering, University of California, Los Angeles, USA}
\icmlaffiliation{sch}{Department of Computer Science, University of California, Los Angeles, USA}

\icmlcorrespondingauthor{Omead Pooladzandi}{opooladz@ucla.edu}

\icmlkeywords{Machine Learning, ICML}

\vskip 0.3in
]




\begin{abstract}
Training machine learning models on massive datasets incurs substantial
computational costs. To alleviate such costs, there has been a sustained effort to develop data-efficient training methods that can carefully select subsets of the training examples that generalize on par with the full training data. However, existing methods are limited in providing theoretical guarantees for the quality of the models trained on the extracted subsets, and may perform poorly in practice. We propose \alg, a method that leverages the geometry of the data to extract subsets of the training examples for efficient machine learning. The key idea behind our method is to dynamically approximate the curvature of the loss function via an exponentially-averaged estimate of the Hessian to select weighted subsets (coresets) that provide a close approximation of the full gradient preconditioned with the Hessian. We prove rigorous guarantees for the convergence of various first and second-order methods applied to the subsets chosen by \alg. Our extensive experiments show that \alg extracts coresets with higher quality compared to baselines and speeds up training of convex and non-convex machine learning models, such as logistic regression and neural networks, by over 2.9x over the full data and 4.5x over random subsets\footnote{ Code is available at \url{ https://github.com/opooladz/AdaCore}}. 
\end{abstract}
\section{Introduction}
Large datasets have been crucial for the success of modern machine learning models. %However, %as datasets grow larger, 
Learning from massive datasets, however, incurs substantial computational costs and becomes very challenging \cite{asi2019importance,strubell2019energy,schwartz2019green}.
Crucially, not all data points are  equally important for learning \cite{birodkar2019semantic,katharopoulos2018not,toneva2018empirical}. While several examples can be excluded from training without harming the accuracy of the final model \cite{birodkar2019semantic,toneva2018empirical}, other points need to be trained on many times to be learned \cite{birodkar2019semantic}.
To improve scalability %and robustness 
of machine learning, it is essential to theoretically understand and quantify the value of different data points on training and optimization.
This %suggests that it would be useful to
allows identifying examples that contribute the most to learning and safely excluding those that are redundant or non-informative.

To find essential data points, recent empirical studies used heuristics such as the fully trained or a smaller proxy model’s uncertainty (entropy of predicted class probabilities) \cite{coleman2020selection}, or forgetting events \cite{toneva2018empirical} to identify examples that frequently transition from being classified correctly to incorrectly. 
Others employ either the gradient norm \cite{alain2015variance,katharopoulos2018not} or the loss \cite{ loshchilov2015online,schaul2015prioritized} to sample important points that reduce variance of  stochastic optimization methods.  
Such methods, however, do not provide any 
theoretical guarantee for the quality of the trained model on the extracted examples.

Quantifying the importance of different data points without training a model to convergence is very challenging. 
First,
the value of each example 
cannot be measured
without updating the model parameters and measuring the loss or accuracy. Second, as the effect of different data points changes throughout training, their value cannot be precisely measured before training converges. Third, to eliminate redundancies, one needs to look at the importance of individual data points as well as the higher-order interactions between data points. Finally, one needs to provide theoretical guarantees for the performance and convergence of the model trained on the extracted data points.

Here, we focus on finding data points that contribute the most to learning and automatically excluding redundancies while training a model.
A practical and effective approach is to carefully select a small subset of training examples that closely approximate the full gradient, i.e., the sum of the gradients over all the training data points.
This idea has been recently employed to find a subset of data points that guarantee convergence of first-order methods to near-optimal solution for training convex models \cite{mirzasoleiman2020coresets}.
However, modern machine learning models are high dimensional and non-convex in nature. 
In such scenarios, subsets selected based on gradient information only capture gradient along the sharp dimensions, and lack diversity within groups of examples
with similar training dynamics. Hence, they representative large groups of examples with a few data points with substantial weights. This introduces a large error in the gradient estimation and result in first-order coresets to perform poorly. 

We propose \textit{ADAptive second-order} \textit{COREsets} (\alg) that incorporates the geometry of the data
to iteratively select weighted subsets (coresets) of training examples that captures the gradient of the loss preconditioned with the Hessian, by maximizing a submodular function.
Such subsets capture the curvature of the loss landscape along different dimensions, and provide convergence guarantees for first and second-order methods.
As a naive use of Hessian at every iteration is prohibitively expensive for overparameterized models, \alg relies on Hessian-free methods to extract coresets that capture the full gradient preconditioned by the Hessian diagonal. Furthermore, \alg exponentially averages first and second-order information in order to smooth the noise in the local gradient and curvature information. 

We first provide a theoretical analysis of our method and prove its convergence for convex and non-convex functions. 
For a $\beta$-smooth and $\alpha$-strongly convex loss function and a subset $S$ selected by \alg that estimates the full preconditioned gradient by an error of at most $\epsilon$, we prove that Newton's method and AdaHessian applied to $S$ with constant stepsize $\eta=\alpha/\beta$ 
converges to a $\beta\epsilon/\alpha$ neighborhood of the optimal solution, in exponential rate. 
For non-convex overparameterized functions such as deep networks, we prove that for a $\beta$-smooth and $\mu$-PL$^*$ loss function satisfying $\|\nabla\mathcal{L}(w)\|^2/2\geq\mu\mathcal{L}(w)$,
(stochastic) gradient descent applied to subsets found by \alg has similar training dynamics to that of training on full data, and converges at a exponential rate.
In both cases, 
\alg leads to a speedup by training on smaller subsets.

Next, we empirically study the examples selected by \alg during training. 
We show that 
as training continues, \alg selects more uncertain or forgettable samples. Hence, \alg effectively determines the value of every learning example, i.e., when and how many times a sample needs to be trained on, and automatically excludes redundant and non-informative instances. Importantly, incorporating curvature in selecting coresets allows \alg to quantify the value of training examples more accurately, and find fewer but more diverse samples than existing methods.

We demonstrate the effectiveness of various first and second-order methods, namely SGD with momentum, Newton's method and AdaHessian, applied to \alg for training models with a convex loss function (logistic regression) as well as models with a non-convex loss functions, namely  ResNet-20, ResNet-18, and ResNet-50, on MNIST, CIFAR10, (Imbalanced) CIFAR100, and  BDD100k \cite{deng2012mnist,cifar10,bdd100k}. Our experiments show that \alg can effectively extract crucial samples for machine learning, resulting in higher accuracy while 
achieving over 2.9x speedup over the full data and 4.5x over random subsets, for training models with convex and non-convex loss functions. 

\section{Related Work}
Data-efficient methods have recently gained a lot of interest. 
However, existing methods often require training the original \cite{ birodkar2019semantic, ghorbani2019data,toneva2018empirical} or a proxy model \cite{coleman2020selection} to convergence, and use
features or predictions of the trained model to find subsets of examples that contribute the most to learning.
While these results empirically confirm the existence of notable semantic redundancies in large datasets \cite{birodkar2019semantic}, such methods cannot identify the crucial subsets before fully training the original or the proxy model on the entire dataset. Most importantly, such methods do not provide any theoretical guarantees for the model's performance trained on the extracted subsets. 

There have been recent efforts to take advantage of the difference in importance among various samples to reduce the variance and improve the convergence rate of stochastic optimization methods. Those that are applicable to overparameterized models employ either the gradient norm \cite{alain2015variance,katharopoulos2018not} or the loss \cite{ loshchilov2015online,schaul2015prioritized} to compute each sample’s importance.  However, these methods do not provide rigorous convergence guarantees and cannot provide a notable speedup. 
A recent study proposed a method, \craig, to find subsets of samples that closely approximate the full gradient, i.e., sum of the gradients over all the training samples \cite{mirzasoleiman2020coresets}. \craig finds the subsets by maximizing a submodular function, and provides convergence guarantees to a neighborhood of the optimal solution for strongly-convex models.
\gradmatch \cite{killamsetty2021grad} proposes a variation to address the same objective using orthogonal matching pursuit (OMP) \cite{killamsetty2021grad}, and 
\glister \citet{killamsetty2020glister} aims at finding subsets that closely approximate the gradient of a held-out validation set.  
However, \textsc{Glister} requires a validation set, and \textsc{GradMatch} 
uses OMP which may return subsets as little as 0.1\% of the intended size. Such subsets are then augmented with random samples. 
In contrast, our method successfully finds subsets of higher quality by preconditioning the gradient by the Hessian information.\looseness=-1
\section{Background and Problem Setting}
Training machine learning models often reduces to minimizing an empirical risk function. Given a not-necessarily convex loss $\Lc$, 
one aims to find model parameter vector $w_*$ in the parameter space $\mathcal{W}$ that minimizes the loss $\Lc$ over the training data:
\vspace{-2mm}
\begin{align}\label{eq:problem}
w_* \in {\arg\min}_{w \in \mathcal{W}} \Lc(w), \quad\quad\quad\\  \Lc(w) := \sum_{i\in V} l_i(w), %+ r(w),
\quad l_i(w)=l(f(x_i,w),y_i).  \nonumber
% \vspace{-3mm}
\end{align}
Here, $V = \{1,\dots,n\}$ is an index set of the training data, $w\in\mathbb{R}^d$ is the  parameters of the model $f$ being trained, and $l_i$ is the loss function associated with training example $i\in V$ with feature vector $x_i\in \mathbb{R}^d$ and label $y_i$. 
We denote the gradient of the loss w.r.t. model parameters by $\mathbf{g} = \nabla \Lc(w)=\frac{1}{|V|}\sum_{i\in V}\frac{\partial l_i}{\partial w}$, %for the gradient 
and the corresponding second derivative (i.e., Hessian) by $\mathbf{H} = \nabla^2\Lc(w) = \frac{1}{|V|}\sum_{i\in V}\frac{\partial^2 l_i}{\partial w_j \partial w_k}$. 

First order gradient methods are popular for solving Problem \eqref{eq:problem}. 
They
start from an initial point $w_0$ and at every iteration $t$, step in the negative direction of the gradient $\mathbf{g}_t$ multiplied by learning rate $\eta_t$.
The most popular first-order method is Stochastic Gradient Descent (SGD) \cite{robbins1951stochastic}: 
\begin{align}\label{eq:gd_update}
    w_{t+1} = w_{t} - \eta_t \mathbf{v}_t, \quad\quad \mathbf{v}_t=\mathbf{g}_t,
\end{align}

SGD is often used with momentum, i.e., $\mathbf{v}_t\!=\!\beta\mathbf{v}_{t-1}\!+\!(1\!-\!\beta)\mathbf{g}_t$ where $\beta\!\in\![0,1]$, accelerating it in dimensions
whose gradients point in the same directions 
and dampening oscillations in dimensions whose gradients change directions \cite{qian1999momentum}.
For larger datasets, mini-batched SGD is used, where $\mathbf{v}_t\!=\!\frac{1}{m}\sum_{j=1}^m  l_{i_t^{(j)}}(w_t)$, where $m$ is the size of the mini-batch of datapoints whose indices $\{i_t^{(1)}, \ldots, i_t^{(m)}\}$ are uniformly drawn with replacement from $V$, at each iteration $t$.\looseness=-1

Second-order gradient methods rely on the geometry of the problem to automatically rotate and scale the gradient vectors, using the curvature of the loss landscape.
In doing so, second-order methods can choose a better descent direction and automatically adjust the learning rate for each parameter.
Hence, second-order methods have superior convergence properties compared to first-order methods.
Newton’s method \cite{bertsekas1982projected} is a classical second order method
that preconditions the gradient vector with inverse of the local Hessian at every iteration, $\mathbf{H}_t^{-1}$:
\begin{equation}\label{eq:newton}
w_{t+1}=w_t-\eta_t \mathbf{H}_t^{-1}\mathbf{g}_t.
\end{equation}
As inverting the Hessian matrix requires quadratic memory and cubic computational complexity, 
several methods approximate Hessian information to significantly reduce time and memory complexity \cite{Nocedal,schaul2013no,martens2015optimizing,xu2020second}. In particular, AdaHessian \cite{yao2020adahessian} directly approximates the diagonal of the Hessian and relies on exponential moving averaging and block diagonal averaging to smooth out and reduce the variation of the Hessian diagonal.  %for parameter update.
\section{\alg: Adaptive Second order Coresets}

The key idea behind our proposed method is to leverage %the knowledge of 
the geometry of the data, precisely the curvature of the loss landscape, to select subsets of the training examples that enable fast convergence. 
Here, we first discuss why coresets that only capture the full gradient perform poorly in various scenarios. Then, we show how to incorporate curvature information in subset selection for training convex and non-convex models with provable convergence guarantees--- ameliorating
problems of first-order coresets. 

\subsection{When First-order Coresets Fail}

First-order coreset methods
iteratively select weighted subsets of training data that closely approximate the full gradient at particular values of $w_t$, e.g. beginning of every epoch \!\cite{killamsetty2021grad,killamsetty2020glister,mirzasoleiman2020coresets}:
\begin{align}
    % \vspace{-4mm}
    \hspace{-2mm}
    S^{*}_t= \!\!\underset{S \subseteq V, \gamma_{t,j} \geq 0 ~\forall j}{\arg\min}|S| \quad \textrm{s.t.} \quad \|\textbf{g}_t-\sum_{j\in S}\gamma_{t,j} \textbf{g}_{t,j}\|\leq \epsilon,\vspace{-4mm}
\end{align} 
where $\textbf{g}_{t,j}$ and $\gamma_{t,j}>0$ are the gradient and the weight of element $j$ in the coreset $S$.
Such subsets often perform poorly for high-dimensional and non-convex functions, due to the following reasons: 
(1) the scale of gradient $\mathbf{g}\in\mathbb{R}^d$ is often different along different dimensions. Hence, the selected subsets estimate the full gradient closely only along dimensions with a larger gradient scale. This can introduce a significant error in the optimization trajectory for both convex and non-convex loss functions; 
(2) the loss functions associated with different data points $l_i$ may have similar gradients but very different curvature properties at a particular $w_t$. Thus, for a small $\delta>0$, the gradients $\nabla l_i(w_t+\delta)$ at $w_t+\delta$ may be totally different than the gradients $\nabla l_i(w_t)$ at $w_t$.
Consequently, subsets that capture the gradient well at %the beginning of epoch
at a particular point during training may not provide a close approximation of the full gradient after a few gradient updates, e.g., mini-batches. %, within the same epoch. 
This often results in inferior performance, particularly when selecting larger subsets for non-convex loss functions;
(3) subsets that only capture the gradient, select one representative example with a large weight from data points with similar gradients at $w_t$. 
%%%%%%%%%%
Such subsets lack diversity %within groups of examples with similar training dynamics.
and cannot distinguish different subgroups of the data.
Importantly, the large weights introduce a substantial error in estimating the full gradient and result in a poor performance, as we show in {Fig. \ref{fig:when_craig_fails} in the Appendix.}

\subsection{Adaptive Second-order Coresets}
To address the above issues,
our main idea is to select subsets of training examples that capture the full gradient preconditioned with the curvature of the loss landscape. 
In doing so, 
we normalize the gradient by multiplying it by the Hessian inverse, $\mathbf{H}^{-1}\mathbf{g}$, before selecting the subsets. 
This allows selecting subsets that (1) 
can capture the full gradient in all dimensions equally well; (2) 
contain a more diverse set of data points with similar gradients, but different curvature properties; and (3) allow adaptive first and second-order methods trained on the coresets to obtain similar training dynamics to that of training on the full data. 

Formally, our goal in \alg is to adaptively find the smallest subset $S \subseteq V$ and corresponding per-element weights $\gamma_j > 0$ that approximates the 
full gradient preconditioned with the Hessian matrix, with an error of at most $\epsilon > 0$ at every iteration $t$, I.e.,:
% \vspace{-2mm}
\begin{align}\label{eq:main}
    S^{*}_t= \!\!\underset{S \subseteq V, \gamma_{t,j} \geq 0 ~\forall j}{\arg\min}&|S|, \quad \textrm{s.t.}\quad\\ 
    &\|\mathbf{H}^{-1}_t \textbf{g}_t-\sum_{j\in S}\gamma_{t,j}\mathbf{H}_{t,j}^{-1} \textbf{g}_{t,j}\|\leq\epsilon,\nonumber
\end{align}
where $\mathbf{H}^{-1}_t \textbf{g}_t$ and $\sum_{j\in S}\gamma_{t,j}\mathbf{H}_{t,j}^{-1} \textbf{g}_{t,j}$
are preconditioned gradients of the full data and the subset $S$. 

\subsection{Scaling up to Over-parameterized Models}\label{sec:diag}

Directly solving the optimization problem \eqref{eq:main} 
requires explicit calculation and storage of the Hessian matrix and its inverse. This is infeasible for large models such as neural networks.
In the following, we first address the issue of calculating the inverse Hessian at every iteration. \!Then, we discuss how to efficiently find a near-optimal subset to estimates the full preconditioned gradient by solving Eq. \!\eqref{eq:main}. \looseness=-1

\paragraph{Approximating the Gradients}\label{sec:approxgrad}
For neural networks, derivative of the loss $\Lc$ w.r.t. the input to the last layer  \cite{katharopoulos2018not, mirzasoleiman2020coresets} or the penultimate layer \cite{killamsetty2021grad} can capture the variation of gradient norm well.
We extend these results (Appendix \ref{proof:boundnormederrornn}) to show that the normed difference preconditioned gradient difference between data points can be approximately efficiently bound by:
\begin{align}
&\| \mathbf{H}_{i}^{-1} \textbf{g}_i - \mathbf{H}_{j}^{-1}  \textbf{g}_j \| \leq  \\ 
&c_1\|\Sigma'_L(z_i^{(L)}) (\mathbf{H}_{i}^{-1}  \textbf{g}_{i})^{(L)}- \Sigma'_L(z_j^{(L)}) (\mathbf{H}_{j}^{-1} \textbf{g}_{j})^{(L)}\|+c_2, \nonumber
\end{align}
where $\Sigma'_L(z_i^{(L)})(\mathbf{H}_{i}^{-1} \textbf{g}_{i})^{(L)}$ is gradient preconditioned by the inverse of the Hessian of the loss  w.r.t. the input to the last layer for data point $i$, and $c_1, c_2$ are constants. Since the upper bound depends on the weight parameters, we need to update our subset $S$ using \alg during the training.

Calculating the last layer gradient often requires only a forward pass, which is as expensive as calculating the loss, and does not require any extra storage. 
For example, having a softmax as the last layer, the gradients of the loss w.r.t. the $i^{th}$ input to the softmax is $p_i-y_i$, where $p_i$ is the $i^{th}$ output the softmax and $y$ is the one-hot encoded label with the same dimensionality as the number of classes. 
Using this low-dimensional approximation $\hat{\textbf{g}_{i}}$ for the gradient $\textbf{g}_{i}$
we can efficiently calculate the preconditioned gradient for every data point. 
For non-convex functions, the local gradient information can be very noisy. To smooth out the local gradient information and get a better approximation of the global gradient, we apply exponential moving average with a parameter $0<\beta_1$ to the low-dimensional gradient approximations:
\begin{equation}\label{eq:g_avg}
\overline{\mathbf{g}}_{t}=\frac{(1-\beta_1)\sum_{i=1}^t \beta_2^{t-i}\hat{\mathbf{{g}_{i}}}}{1-\beta_2^t}.
\end{equation}
\paragraph{Approximating the Hessian Preconditioner}
Since it is infeasible to calculate, store, and invert the full Hessian matrix every iteration, we use an inexact Newton method, where an approximate Hessian operator is used instead of the full Hessian.
To efficiently calculate the Hessian diagonal, we first use the Hessian-Free method \cite{yao2018inexact} to compute the multiplication between Hessian $\mathbf{H}_t$ and a random vector $z$  with Rademacher distribution. To do so, we backpropagate on the low-dimensional gradient estimates multiplied by $z$ to get $\mathbf{H}_tz=\partial \hat{\mathbf{g}}_t^T z /\partial w_t$.
Now, we can use the Hutchinson's method of obtains a stochastic estimate of the diagonal of the Hessian matrix as follows: 
\begin{align}\label{eq:hf}
\text{diag}(\mathbf{H}_t)=\mathbb{E}[z\odot (\mathbf{H}_tz)], 
\end{align} without having to form the Hessian matrix explicitly \cite{BEKAS20071214}.
The diagonal approximation has the same convergence rate as using Hessian for strongly convex, and strictly smooth functions (Proof in Appendix \ref{proof:diagconv}). Nevertheless, our method can be applied to general machine learning problems, such as deep networks and regularized classical methods (e.g., SVM, LASSO), which are strongly-convex.
To smooth out the noisy local curvature and get a better approximation of the global Hessian information, we apply an exponential moving average with parameter $0<\beta_2<1$ to the Hessian diagonal estimate in Eq. \eqref{eq:hf}: 
% \vspace{-2mm}
\begin{equation}\label{eq:h_avg}
\overbar{\mathbf{H}}_{t}=\sqrt{\frac{(1-\beta_2)\sum_{i=1}^t \beta_2^{t-i} \text{diag}(\mathbf{H}_i)\text{diag}(\mathbf{H}_i)}{1-\beta_2^t}}.
\end{equation}
Using exponentially averaged gradient and Hessian approximations in Eq. \eqref{eq:g_avg}, and \eqref{eq:h_avg}, the preconditioned gradients in Eq. 
\eqref{eq:main} can be approximated as follows:
\begin{align}\label{eq:H_sub}
    S^{*}_t= \!\!\underset{S \subseteq V, \gamma_{t,j} \geq 0 ~\forall j}{\arg\min}&|S|, \quad \textrm{s.t.}\quad\\
    &\|\overbar{\mathbf{H}}_t^{-1} \overline{\textbf{g}}_t-\sum_{j\in S}\gamma_{t,j}\overbar{\mathbf{H}}_{t,j}^{-1} \overline{\textbf{g}}_{t,j}\|\leq\epsilon\nonumber.
\end{align}
Next, we discuss how to 
efficiently find near-optimal weighted subsets that closely approximate the full preconditioned gradient by solving Eq. \eqref{eq:main}.
\subsection{Extracting Second-order Coresets}\label{sec:alg}
The subset selection problem \eqref{eq:main} is NP-hard \cite{natarajan1995sparse}. %in general.
However, it can be considered as a special case of the sparse vector approximation problem that has been studied in the literature, including convex optimization formulations—e.g. basis pursuit \cite{chen2001atomic}, sparse projections \cite{pilanci2012recovery,kyrillidis2013sparse}, LASSO \cite{tibshirani1996regression},  
and compressed sensing \cite{donoho2006compressed}. These methods, however, are expensive to solve and often require tuning regularization coefficients and thresholding to ensure cardinality constraints. 
More recently, the connection between sparse modeling and \textit{submodular}\footnote{A set function $F:2^V \rightarrow \R^+$ is submodular if $F(S\cup\{e\}) - F(S) \geq F(T\cup\{e\}) - F(T),$ for any $S\subseteq T \subseteq V$ and $e\in V\setminus T$.} optimization have been demonstrated \cite{elenberg2018restricted, mirzasoleiman2020coresets}.
The advantage of submodular optimization is that a fast and simple greedy algorithm often provides a near-optimal solution. 
Next, we briefly discuss how submodularity can be used to find a near-optimal solution for Eq. \eqref{eq:main}. We build on the recent result of \cite{mirzasoleiman2020coresets} that showed that the error of estimating an expectation by a weighted sum of a subset of elements is upper-bounded by a submodular facility location function.
In particular, via the above result, we get:
\vspace{-2mm}
\begin{align}\label{eq:min_upper}
\min_{S\subseteq V} \| \overbar{\mathbf{H}}_t^{-1}\overline{\mathbf{g}}_t -  
\sum_{j \in S}  & \gamma_{t,j}\overbar{\mathbf{H}}_{t,j.}^{-1}\overline{\mathbf{g}}_{t,j}^{}~\| \\
&\leq \sum_{i\in V} \min_{j \in S} \| \overbar{\mathbf{H}}_{t,i.}^{-1}\overline{\mathbf{g}}_{t,i}^{} -  \overbar{\mathbf{H}}_{t,j.}^{-1}\overline{\mathbf{g}}_{t,j}^{}\|.\nonumber
\vspace{-6mm}
\end{align}
Setting the upper bound in the right-hand side of Eq. \eqref{eq:min_upper} to be less than $\epsilon$ results in the smallest weighted subset $S^*$ that approximates full preconditioned gradient by an error of at most $\epsilon$, at iteration $t$.
Formally, we wish to solve the following optimization problem:
\begin{align}
    S^* \in &{\arg\min}_{S\subseteq V} |S|, \quad \text{s.t.} \quad\label{eq:L}\\ &L(S)= %\epsilon,\label{eq:L}\\
    % &L(S)\!=\!
    \sum_{i\in V} \min_{j \in S} \| \overbar{\mathbf{H}}_{t,i.}^{-1}\overline{\mathbf{g}}_{t,i}^{} -  \overbar{\mathbf{H}}_{t,j.}^{-1}\overline{\mathbf{g}}_{t,j}^{}\| \leq \epsilon, \nonumber
\end{align}
By introducing a phantom example
$e$, we can turn the minimization problem \eqref{eq:L} into the following submodular cover problem, with a facility location objective $F(S)$:
\begin{align}\label{eq:cover}
    S^* \in  \underset{S\subseteq V}{\arg\min} &
    |S|, \quad \text{s.t.}\\ \quad &F(S)=C_1 - L(S \cup \{e\}) \geq C_1-\epsilon\nonumber,
\end{align}
where $C_1=L(\{e\})$ is a constant upper-bounding the value of $L(S)$.
The subset $S^*$ obtained by solving the maximization problem \eqref{eq:cover} is the medoid of the preconditioned gradients, and the weights $\gamma_j$ are the number of elements that are closest to the medoid $j\in S^*$, i.e. $\gamma_j=\!\sum_{i\in V} \mathbb{I}[j=\min_{s \in S} \| \overbar{\mathbf{H}}_{t,i}^{-1}\overline{\mathbf{g}}_{t,i}^{} \!\!-\! \overbar{\mathbf{H}}_{t,s}^{-1}\overline{\mathbf{g}}_{t,s}^{} \|]$.
For the above submodular cover problem, the classical greedy algorithm provides a logarithmic approximation guarantee $|S| \leq \big(1+ \ln (\max_e F(e|\emptyset))\big) |S^*|$ \cite{wolsey1982analysis}. The greedy algorithm starts with the empty set $S_0=\emptyset$, and at each iteration $t$, it chooses an element $e\in V$ that maximizes the marginal utility $F(e|S_{t})=F(S_{t}\cup\{e\}) - F(S_{t})$. Formally,
$S_t = S_{t-1}\cup\{{\arg\max}_{e\in V} F(e|S_{t-1})\}$.
The computational complexity of the greedy algorithm is $\mathcal{O}(nk)$.
However, its complexity can be reduced to $\mathcal{O}(|V|)$ using stochastic methods \cite{mirzasoleiman2015lazier}, and can be further improved using lazy evaluation \cite{minoux1978accelerated} and distributed implementations \cite{mirzasoleiman2013distributed}.  The pseudocode can be found in Alg. \ref{alg:greedy} in  Appendix \ref{appx:alg}.

\paragraph{One coreset for convex functions}\label{dis:onecoresetcvx}
For convex functions, normed gradient differences between data points can be efficiently upper-bounded by the normed difference between feature vectors \cite{allen2016exploiting,hofmann2015variance, mirzasoleiman2020coresets}. We apply a similar idea to upper-bound the normed difference between preconditioned gradients. This allows us to find one subset before the training. See proof in Appendix \ref{proof:boundnormerror}. 

\subsection{Convergence Analysis}\label{sec:convergence}
Here, we analyze the convergence rate of first and second order methods applied to the weighted subsets $S$ found by \alg. By minimizing Eq. \eqref{eq:cover} at every iteration $t$, \alg finds subsets that approximate full preconditioned gradient by an error of at most $\epsilon$, i.e. $\|\mathbf{H}^{-1}_t \textbf{g}_t-\sum_{j\in S}\gamma_{t,j}\mathbf{H}_{t,j}^{-1} \textbf{g}_{t,j}\|\leq\epsilon$. This allows us to effectively analyze %the convergence, i.e., 
the reduction in the value of the loss function $\Lc$ at every iteration $t$. Below, we discuss the convergence of a first and second-order gradient method%, namely Newton's method and gradient descent,
applied to subsets extracted by \alg.

\textbf{Convergence for Newton's Methods and AdaHessian}
We first provide the convergence analysis for the case where the %sum 
function $\Lc$ in Problem (\ref{eq:problem}) is strongly convex, i.e. there exist a constant $\alpha>0$ such that $\forall w,w'\in \mathbb{R}^d$ we have $\Lc(w) \geq \Lc(w') + \langle \nabla \Lc(w'), w-w' \rangle + \frac{\alpha}{2} \| w'-w \|^2$, and each component function has a Lipschitz gradient, i.e. $ \forall w \in \mathcal{W}$ we have $\| \nabla \Lc(w) - \nabla \Lc(w') \| \leq\! \beta \| w-w'\|$. We get the following results %about the iterates generated 
by applying Newton's method and AdaHessian to the weighted subsets $S$ extracted by \alg.

\begin{restatable}{theorem}{newtonrestate}
\label{thm:newton}
    Assume that $\Lc$ is $\alpha$-strongly convex and  $\beta$-smooth.
    Let $S$ be a weighted subset %of size $r$ 
    obtained by \alg that estimate the %general descent direction 
    preconditioned gradient
    by an error of at most $\epsilon$ at every iteration $t$, i.e., $\|\mathbf{H}^{-1}_t \textbf{g}_t-\sum_{j\in S}\gamma_{t,j}\mathbf{H}_{t,j}^{-1} \textbf{g}_{t,j}\|\leq\epsilon$. Then with learning rate $\alpha/\beta$, Newton's method with update rule of Eq. \eqref{eq:newton} 
    applied to the subsets has the following convergence behavior: % as follows:
    % \vspace{-2mm}
    \begin{align}
            \Lc(w_{t+1}) - \Lc(w_t) \leq -\frac{\alpha^{3}}{2\beta^{4}} (\|\mathbf{g}_t\|-\beta\epsilon)^2.
    \end{align} 
In particular, the algorithm converges to a $\beta\epsilon/\alpha$-neighborhood of the optimal solution $w_*$. 

\end{restatable}

\begin{corollary}\label{thm:adahessian}
    For an $\alpha$-strongly convex and $\beta$-smooth loss $\Lc$, AdaHessian with Hessian power $k$, applied to subsets found by \alg converges to a $\beta\epsilon/\alpha$-neighborhood of the optimal solution $w_*$, 
    and satisfies:%\vspace{-2mm}
        \begin{align}
            \Lc(w_{t+1}) - \Lc(w_t) \leq -\frac{\alpha^{k+2}}{2\beta^{k+3}} (\|\mathbf{g}_t\|-\beta\epsilon)^2.
    \end{align}
\end{corollary}
The proofs 
can be found in Appendix \ref{proof:4.1}.

\textbf{Convergence for (S)GD in Over-parameterized Case}
Next, we discuss the convergence behavior of gradient descent applied to the subsets found by \alg. In particular,
we build upon the recent results of \cite{liu2020toward} that guarantees convergence for first-order methods on a broad class of general over-parameterized non-linear systems, including neural networks for which the tangent kernel, defined as $\mathbf{J}^T\mathbf{J}$ are not close to constant, but satisfy the Polyak-Lojasiewicz (PL) condition. Where $\mathbf{J}=\partial f/\partial w$ is the Jacobian of the function $f$ with respect to the parameters $w$.
A loss function $\Lc$ is $\mu$-PL$^*$ on a set $\mathcal{W}$, if $\frac{1}{2}\| \nabla \Lc(w)\|^2\geq \mu \Lc(w), \forall w\in \mathcal{W}$.
\begin{restatable}{theorem}{plrestate}
\label{thm:pl}
    Assume that the loss function $\Lc(w)$ is $\beta$-smooth, and $\mu$-PL$^*$ on a set $\mathcal{W}$, and $S$ is a weighted subset %of size $r$
    obtained by \alg that estimates the preconditioned gradient by an error of at most $\epsilon$, i.e., $\|\mathbf{H}^{-1}_t \textbf{g}_t-\sum_{j\in S}\gamma_{t,j}\mathbf{H}_{t,j}^{-1} \textbf{g}_{t,j}\|\leq\epsilon$. 
    Then with learning rate $\eta$, gradient descent with update rule of Eq. \eqref{eq:gd_update} %with $m_t=\mathbf{g}$ and $v_t=1$
    applied to the subsets have the following convergence behavior at iteration $t$:
\vspace{-1mm}
\begin{align}
    \Lc(w_{t}) \leq (1-\frac{\eta\mu\alpha^2} {\beta^2})^{t} \Lc(w_0) - \frac{\eta\alpha^2} {2\beta^2}(\beta^2\epsilon^2-2\beta\epsilon \nabla_{\max}), 
\end{align} where $\alpha$ is the minimum eigenvalue of all Hessian matrices during training, and $\nabla_{\max}$ is an upper bound on the norm of the gradients.
\end{restatable}
\begin{restatable}{theorem}{plsgdrestate}
\label{thm:pl-sgd}
    Under the same assumptions as in Theorem \ref{thm:pl}, for mini-batch SGD with
    %For any 
    mini-batch size $m \in \mathbb{N}$, the mini-batch SGD with update rule Eq. \eqref{eq:gd_update}, with learning rate $\eta = \frac{m}{\beta(m-1)}$,
    applied to the subsets have the following convergence behavior:
\vspace{-1mm}
\begin{align}
        \mathbb{E}[\Lc(w_{t})] \leq (1-\frac{\eta\mu \alpha^2}{2\beta})^t \mathbb{E}[\Lc(w_0)] - \frac{\alpha^2\eta}{2\beta}(  \beta\epsilon^2 -2 \epsilon\nabla_{\max})\label{eq:adacore_convergence_sgd}
\end{align} where $\alpha$ is the minimum eigenvalue of all Hessian matrices during training, and $\nabla_{\max}$ is an upper bound on the norm of the gradients, and the expectation is taken w.r.t. the randomness in the choice of mini-batch.

\end{restatable}
The proofs can be found in Appendix  \ref{proof:4.3}.

We show an {exponential} convergence 
for GD (Theorem \ref{thm:pl}) and SGD (Theorem \ref{thm:pl-sgd}) under the $\mu$-PL$^*$ condition, as well as for second order methods (Theorems \ref{thm:newton}, \ref{thm:adahessian}), under $\alpha$-strongly convex and $\beta$-smooth assumptions on the loss.
\section{Experiments}%\vspace{-2mm}
In this section, we evaluate the effectiveness of \alg, 
by answering the following questions:
(1) how does the performance of various first and second-order methods compare when applied to subsets found by \alg vs. the full data and baselines; (2) how effective is \alg for extracting crucial subsets for training convex and non-convex over-parameterized models with different optimizers; and (3) how does \alg perform in eliminating redundancies 
%in the training data and extracting essential subsets for learning. 
and enhancing diversity of the selected elements. %\ba{better?}

\textbf{Baselines} In the convex setting, we compare the performance of \alg with \craig \cite{mirzasoleiman2020coresets} that extracts subsets that approximate the full gradient, as well as Random subsets. For non-convex experiments, we additionally compare \alg with \gradmatch and \glister \cite{killamsetty2021grad,killamsetty2020glister}. 
For \alg and \craig, we use the gradient w.r.t the input to the last layer, and for \glister and \gradmatch we use the gradient w.r.t the penultimate layer, as specified by the methods.
In all cases, we select subsets separately from each class proportional to the class sizes,
and 
train on
the union of the subsets.
We
report average test accuracy across 3 trials in all experiments. 

\subsection{Convex Experiments}
In our convex experiments, we apply \alg to select a coreset to classify the Ijcnn1 dataset using L2-regularized logistic regression: $f_i(x) = ln(1-\text{exp}(-w^{T}x_yy_i)) + 0.5 \mu w^{T}w$. Ijcnn1 includes 49,990 training and 91,701 test data points of 22 dimensions, from 2 classes with 9-to-1 class imbalance ratio. In the convex setting, we only need to calculate the curvature once to find one \alg subset for the entire training. Hence, we utilize the complete Hessian information, computed analytically, as discussed in Appendix \ref{appx:hessian}.
We apply an exponential decay learning schedule $\alpha_k = \alpha_0 b^k$ with learning rate parameters $\alpha_0$ and $b$. For each model and method (including the random baseline) we tuned the parameters via a search and reported the best results.\looseness=-1

\textbf{\alg achieves smaller loss residual with a speedup}
Figure \ref{fig:grad_diff} compares the loss residual for SGD and Newton's method  applied to coresets of size 10\% extracted by \alg (blue), \craig (orange), and random (green) with that of full dataset (red).
We see that \alg effectively minimizes the training loss, achieving a better loss residual than \craig and random sampling. In particular, \alg matches the loss achieved on the full dataset with more than a 2.5x speedup for SGD and Newton's methods. We note that training on  random 10\% subsets of the data cannot effectively minimize the training loss.
We show the superior performance of training with SGD on subsets of size 10\% to 90\% found with \alg vs \craig in Appendix Fig. \ref{fig:normed_grad_diff_subsets}.\looseness=-1 

\textbf{\alg better estimates the full gradient}
Fig. \ref{fig:grad_diff_2} shows the normalized gradient difference between gradient of the full data vs. weighted gradient of subsets of different sizes obtained by \alg vs \craig and Random, at the end of training by each method. We see that by considering curvature information, \alg obtains a better gradient estimate than \craig and Random subsets.
\begin{figure}[t]
    \centering
% 	\vspace{-4mm}
    \subfloat[Ijcnn1 SGD ]{
	\includegraphics[width=.22\textwidth,height=.15\textwidth,trim=10mm 0 10mm 0]{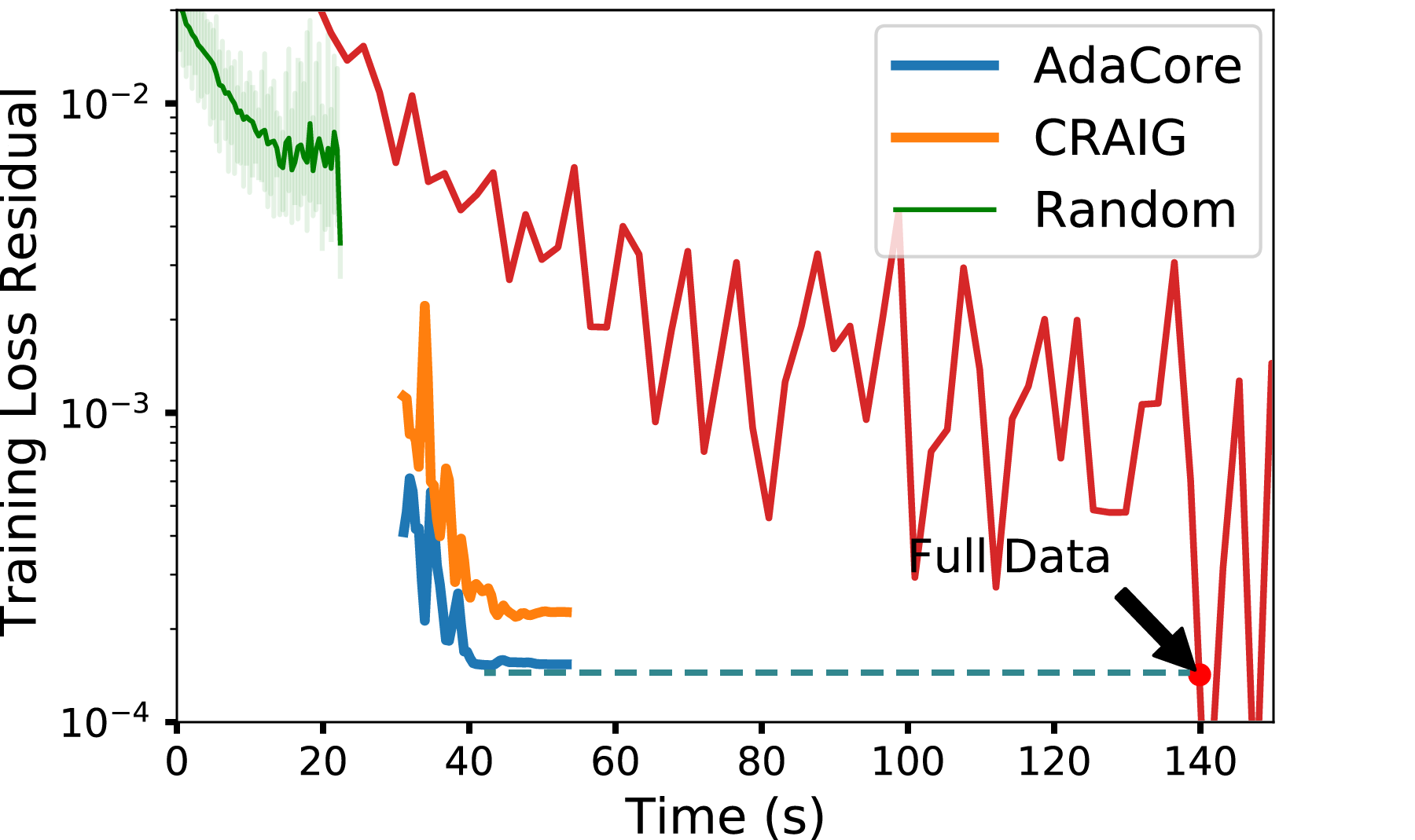}
    }
    \subfloat[Ijcnn1 Newton]{
    \includegraphics[width=.22\textwidth,height=.15\textwidth,trim=10mm 0 10mm 0]{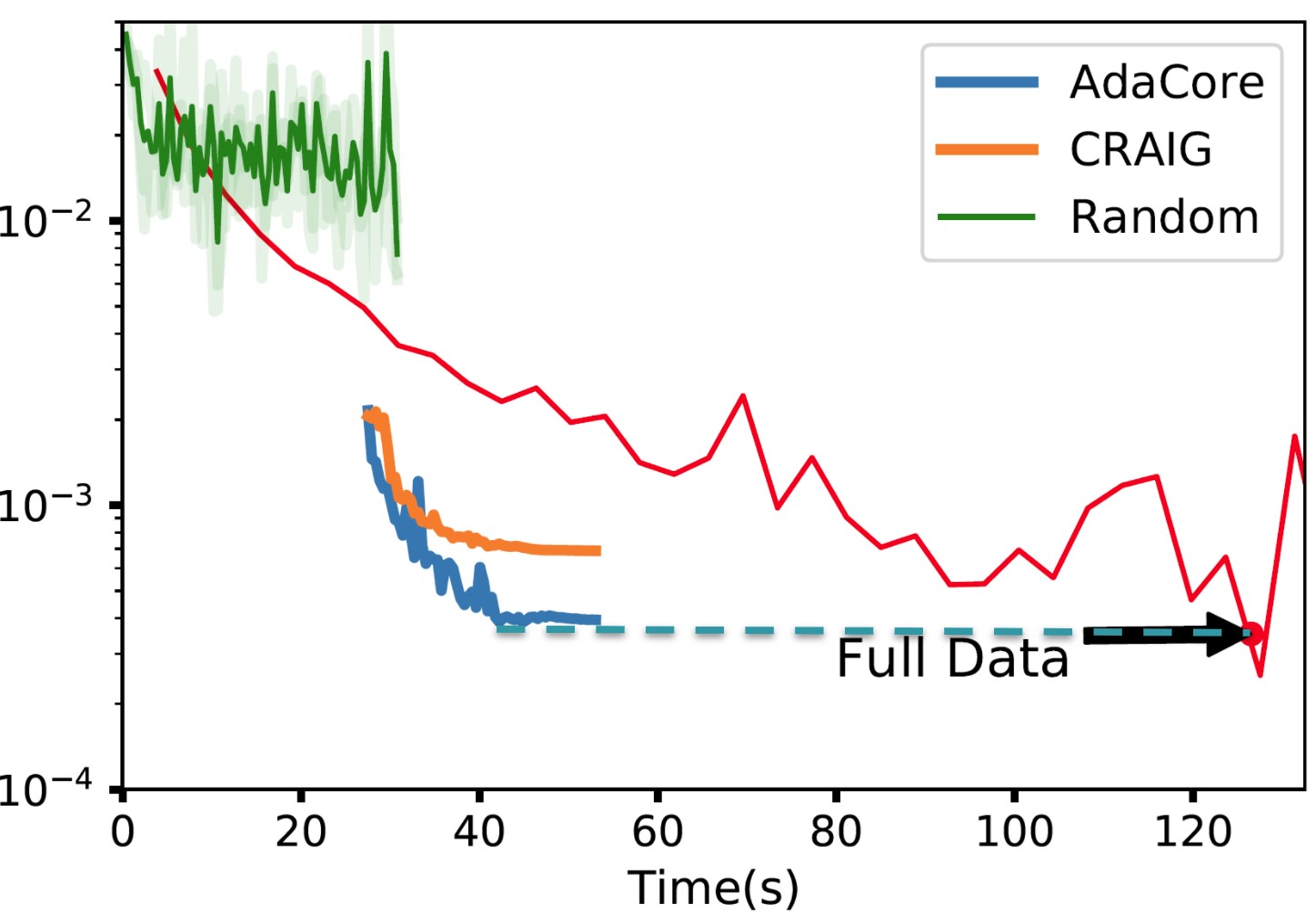}
    }
    \vspace{-2mm}
    \caption{Loss residual of SGD and Newton's method for training Logistic Regression on Ijcnn1. Comparing \alg (blue), \craig (orange) and random subsets (green) of size 10\% vs. entire data (red dot). 
    \alg achieves 2.5x speedup for training with SGD and Newton's method. 
    }
    \vspace{-2mm}
    \label{fig:grad_diff}
\end{figure}
\begin{figure}[t]
    \centering
	\vspace{-4mm}
    \subfloat[Ijcnn1 SGD ]{
	\includegraphics[width=.24\textwidth,height=.15\textwidth,trim=0 0 0 10mm 0]{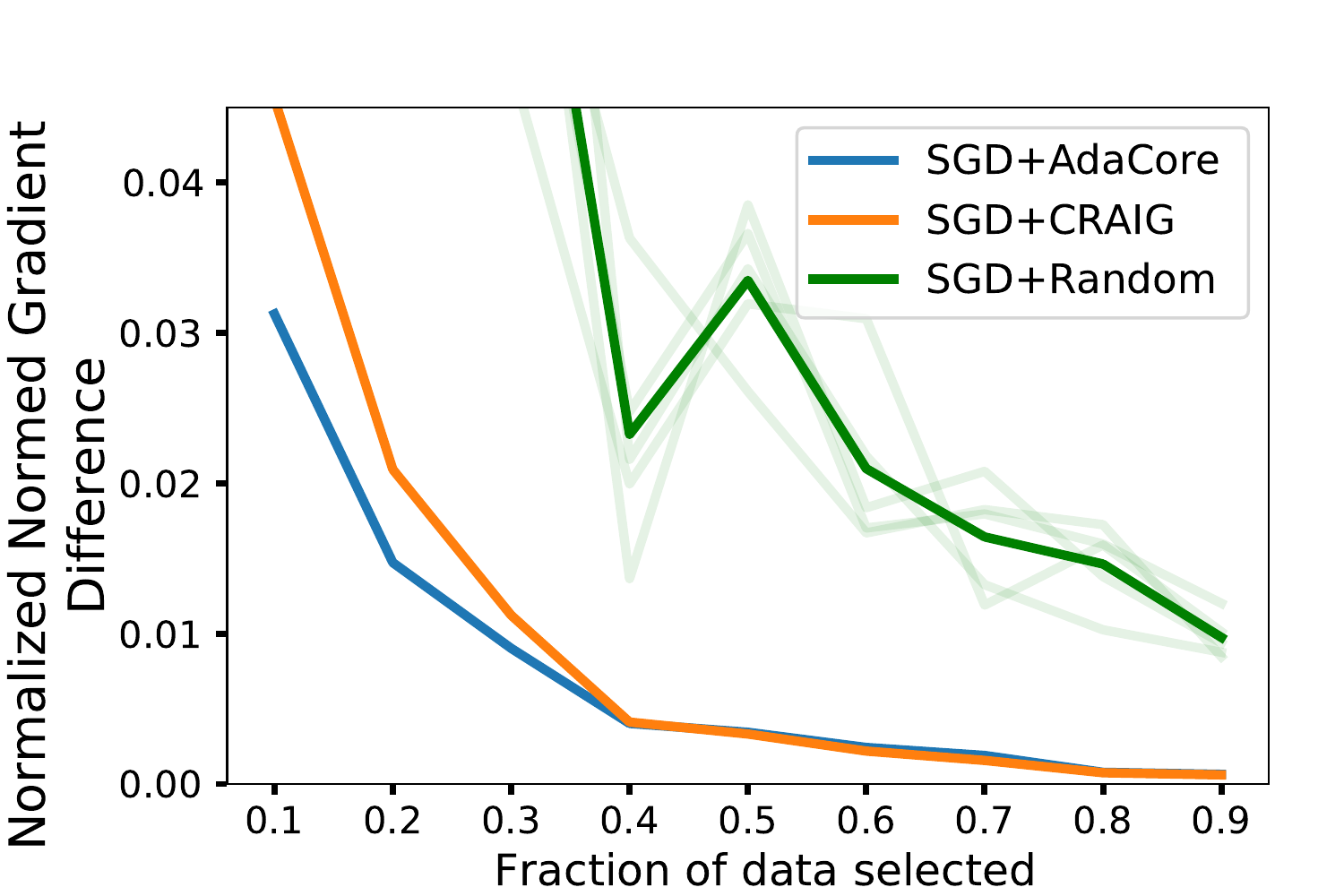}
    }
    \subfloat[Ijcnn1 Newton]{
    \includegraphics[width=.24\textwidth,height=.15\textwidth,trim=10mm 0 0 10mm 0]{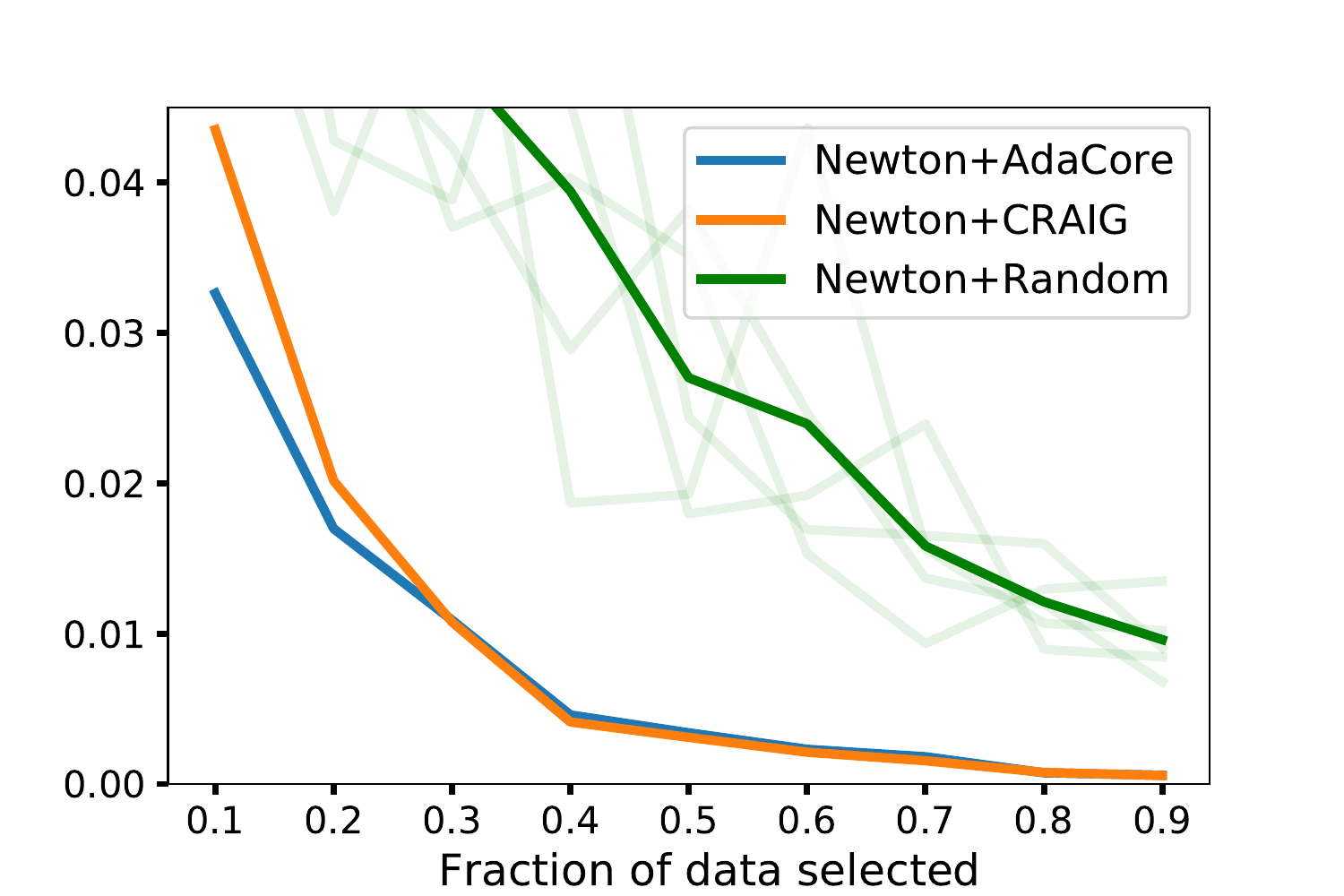}
    }\vspace{-2mm}
    \caption{Normalized gradient difference  between subsets of various sizes found by \alg (blue), \alg (orange), Random (green) vs full data, when training Logistic Regression with SGD and Newton on Ijcnn1.
    \alg has a smaller gradient error at the end of training.
    }
    \vspace{-4mm}
    \label{fig:grad_diff_2}
\end{figure}

\begin{figure*}[t]
% \vspace{-5mm}
    \centering
%  0]{Fig/test_accs_B.png}
% %     }
    \subfloat[Test Accuracy\label{subfig:R18_acc}]{
	\includegraphics[width=.22\textwidth,trim=10mm 5mm 0 0]{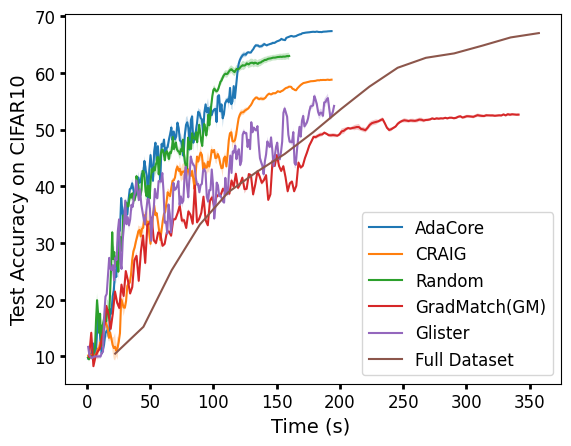}
    }    \hspace{3mm}
    \subfloat[Distribution of selected points \label{subfig:R18_hist}]{
    \includegraphics[width=.25\textwidth,trim=5mm 5mm 10mm 0]{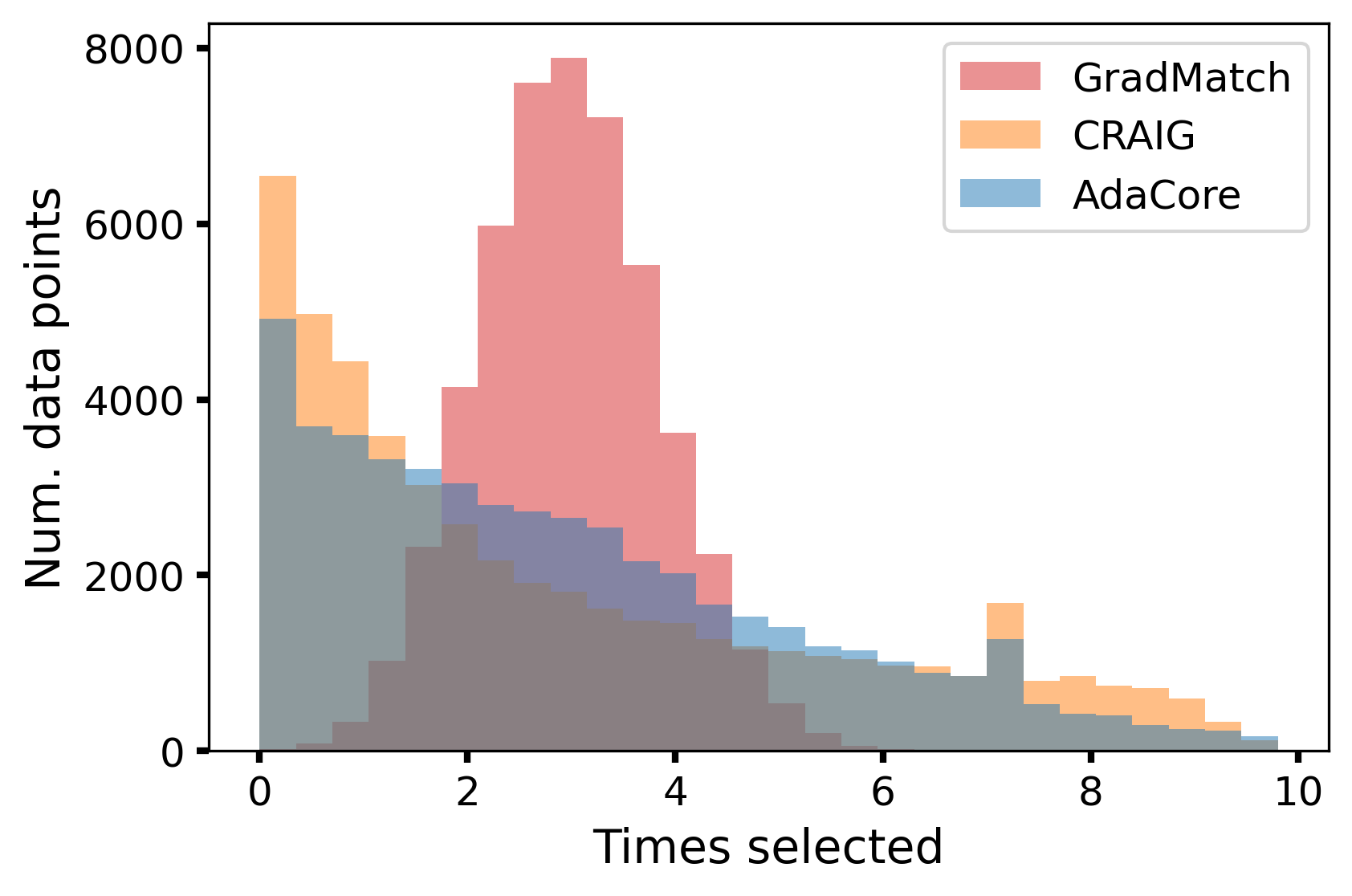}
    }\hspace{9mm}
    \subfloat[Forgetting vs class ranking \label{subfig:R18_forget}]{
    \includegraphics[width=.21\textwidth,trim=10mm 5mm 10mm 0]{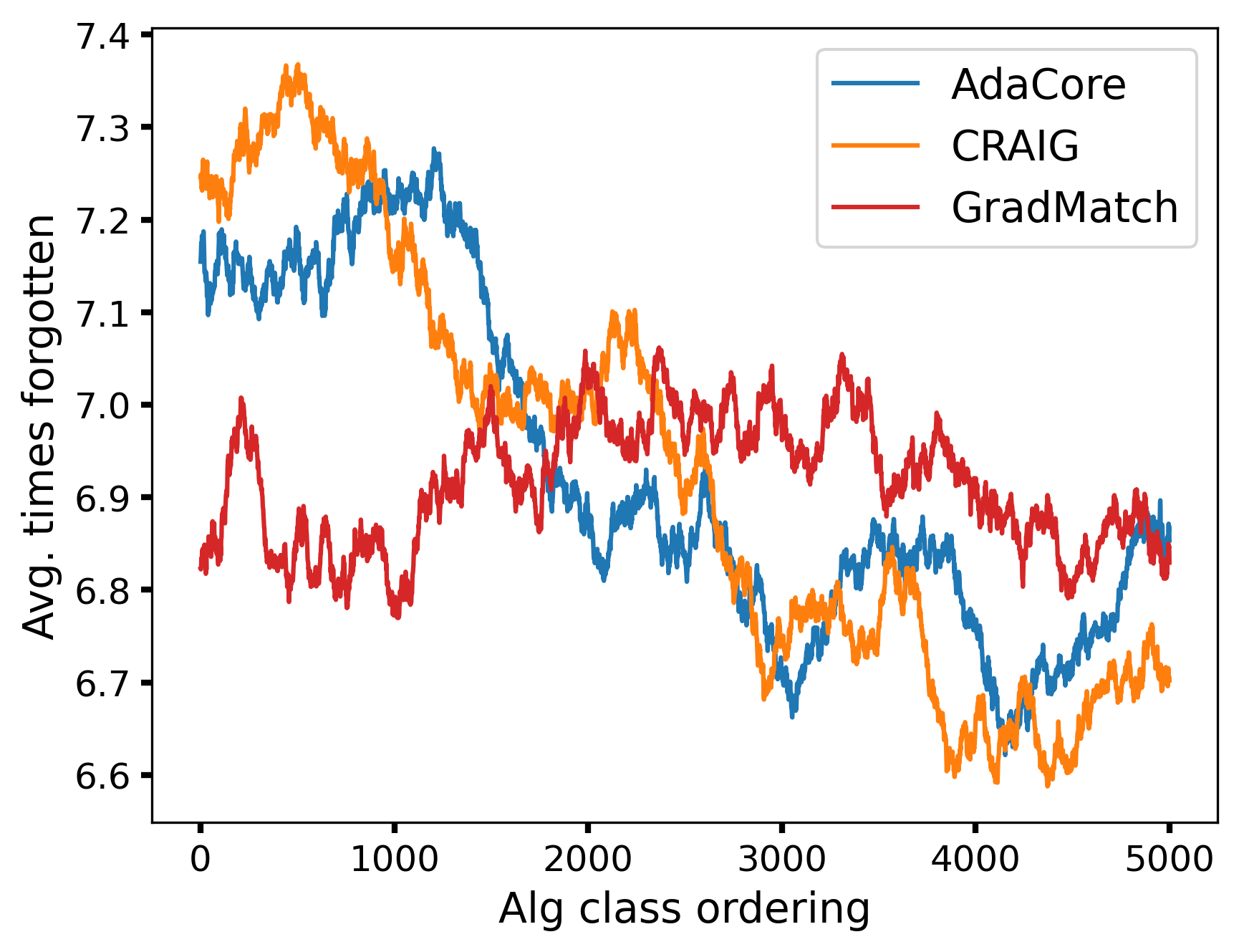}
    }
    \vspace{-4mm}
    \caption{
    Training ResNet-18 on subsets of size $S$=1\% selected every $R$=1 epoch, with \alg, \craig, \glister, \gradmatch and Random for 200 epochs vs. full for 15 epochs. 
    (a) \alg outperforms baselines by providing 2x speedup over full, and more than 4.5x speedup over Random.
 (b) Histograms of the number of times a point is selected by \alg, \craig and \gradmatch. \alg selects a more diverse set of examples compare to \craig, and \gradmatch augments several randomly selected examples. 
 (c)Forgetting scores for examples of a class sorted by \alg at the end of training.
 \alg priorities less forgettable examples compared to \craig.
 }
 \vspace{-4mm}
    \label{fig:craig_vs_ada}
    % \vspace{-3mm}
    \end{figure*}

\begin{table}[t]
\small
\centering
\caption{Training ResNet20 using AdaHessian and SGD+momentum on coresets of size 1\% selected by 
different methods from CIFAR10. Percent of full data selected during entire training is shown (in parentheses). Using $b_H$=64, \alg 
achieves up to 16.8\% higher accuracy, while selecting a smaller fraction of data points. 
Exponential averaging of gradient and Hessian, \!and a smaller $b_H$ helps.\looseness=-1
}
\vspace{-2mm}
\resizebox{\columnwidth}{!}{
    \begin{tabular}{cccc}
        \hline
        {\color[HTML]{000000} \textbf{}} & {\color[HTML]{333333} \textbf{AdaHessian}}  & {\color[HTML]{333333} \textbf{SGD+Momentum}} \\ \hline
        \multicolumn{1}{l|}{Random } &   $59.1\% \!\pm 2.8 (87\%)$ &    $45.9\% \!\pm 2.5 (87\%)$\\
        \multicolumn{1}{l|}{ \craig} &   $59.5\%  \!\pm 2.8 (74\%)$  &  $ 43.6\% \!\pm 1.6 (75\%)$ \\%\hline
        \multicolumn{1}{l|}{ \gradmatch} &   $57.5\% \!\pm 1.3 (74\%)$  &   $49.4\% \!\pm 1.6 (74\%)$ \\
        \multicolumn{1}{l|}{ \glister} &  $37.5\% \!\pm 1.3 (74\%)$ &   $38.6\% \pm 1.6 (74\%)$ \\\hline
        % \multicolumn{1}{l|}{\textbf{\alg GN} 1\%}   &  \textbf{60.2\%} \pm 2.8 (\textbf{ 73\%})  &  \textbf{53\%} \pm 2.5 ( \textbf{72\%}) &  \textbf{50.4\%} \pm 1.7 ( \textbf{74\%})\\
        \multicolumn{1}{l|}{\alg \!(no avg) }   &  $58.4\% \!\pm 0.2 (73\%)$  &  $51.5\% \!\pm 1.1 (74\%)$\\
        \multicolumn{1}{l|}{\alg \! (avg g)  }   &  $59.8\% \!\pm 0.5 (73\%)$   &  $53.2\% \pm 1.1 (74\%)$\\
        \multicolumn{1}{l|}{\alg  \!(avg H) }   & $60.2\% \!\pm 0.5 (73\%)$   &  $54.4\% \!\pm 1.1 (74\%)$\\
        \multicolumn{1}{l|}{\textbf{\alg   } }   & $ \textbf{60.2\%} \!\pm 0.5 (\textbf{73\%}) $  &  $\textbf{55.4\%} \!\pm 1.1 ( \textbf{74\%})$\\
        \hline
        \multicolumn{1}{l|}{\alg $b_h\!\!=\!\!512$ }   &  $57.2\% \pm 0.5 (73\%)$   &  $52.4\% \pm 1.1 (74\%)$\\        
    \end{tabular}
    }
    \label{table:resnet20_extended}
    \vspace{-5mm}
\end{table}

% \vspace{-3mm}
\subsection{Non-Convex Experiments}

\textbf{Datasets} We use CIFAR10 (60k points from 10 classes) , 
class imbalanced version of CIFAR10 (32.5k points from 10 classes) and CIFAR100 (32.5k points from 100 classes) \cite{cifar10}, BDD100k (100k points from 7 classes) \cite{bdd100k}. The results on MNIST (70k points from 10 classes) \cite{deng2012mnist} can be found in Appendix \ref{exp:mnist}. 
Images are normalized to [0,1] by division with 255. 

\vspace{-1mm}
\textbf{Models and Optimizers} We train ResNet-20 and
ResNet-18 \cite{he2016deep}, with convolution, average pooling and dense layers with softmax outputs and 
weight decay of $10^{-4}$.
We use a batch size of 256 in all experiments (except Table \ref{table:batch}, Fig. \ref{subfig:bdd}), and train using SGD with momentum of 0.9 (default), or AdaHessian. For training, we use standard learning rate scheduler for ResNet starting with 0.1 and exponentially decaying by factor 0.1 at epochs 100 and 150.
We used linear learning rate warm-up for the first 20 epochs to prevent weights from diverging when training with subsets.
All experiments were ran on a 2.4GHz CPU and RTX 2080 Ti GPU.

\textbf{Calculating the Curvature} To calculate the Hessian diagonal using Eq. \eqref{eq:hf}, we use a batch size of $b_H\!=\!64$ to calculate the expected Hessian diagonal over the training data. We observed that a smaller batch size provides a higher quality Hessian compared to larger batch sizes, \!as shown in Table \ref{table:resnet20_extended}.

\textbf{Baseline Comparison and Ablation Study}
Table \ref{table:resnet20_extended} shows the accuracy of training ResNet-20, using SGD with momentum of 0.9 and AdaHessian,  for 200 epochs on $S$=1\% subsets of CIFAR-10 chosen every $R$=1 epoch by different methods.
For SGD+momentum, \alg outperforms \craig by 12\%, Random by 10\%, \gradmatch by 6\%, and \glister by 16.8\%. 
Note that in total, \alg selects 74\% of the dataset during the entire training process, whereas Random visits 87\%. Thus, \alg effectively selects subsets contributing the most to generalization. We see that the accuracy gap between the baselines and \alg shrinks when applying more powerful optimizers such as AdaHessian. 
Table \ref{table:resnet20_extended} also shows the effect of exponential averaging of gradients and Hessian diagonal, and larger batch sizes for calculating the Hessian diagonal $b_H$. We see that exponential averaging help \alg achieving better performance, and smaller $b_H$ provides better results.

Fig. \ref{subfig:R18_acc} compares the performance of ResNet-18 
on 1\% subsets selected from CIFAR-10 with different methods. 
We compare the performance of training on \alg, \craig, \gradmatch, \glister, and Random subsets for 200 epochs, with 
training on full data for 15 epochs. This is the number of iterations required for training on the full data to achieve a comparable performance to that of \alg subsets.
We see that training on \alg coresets achieves a better accuracy 2.5x faster than training on the full dataset, and more than 4.5x faster than the next best subset selection algorithm for this setting
(\textit{c.f.} Fig. \ref{fig:repeat_craig_vs_ada_1000} in Appendix for complete results).

\begin{table}[b]
\vspace{-5mm}
\caption{Test accuracy and percent of full data selected (in parentheses), when selecting $S$=1\% coresets every $R$ epochs from Imbalanced CIFAR-10 to train ResNet18.
}\vspace{-2mm}
\label{table:cifar_imb_sgd_mom}
\begin{small}
\resizebox{\columnwidth}{!}{
\begin{tabular}{p{1cm}|lll}
\hline
 & {\color[HTML]{333333} $S$=$1\%$, $R$=20} & $S$=$1\%$, $R$=10 & $S$=$1\%$, $R$=5 \\ \hline
AdaCore & $\textbf{57.3\%} (\textbf{5\%})$ & $\textbf{57.12} (\textbf{9.5\%})$ & $\textbf{60.2\%} (\textbf{14.5\%})$ \\
\craig & $48.6\% (8\%)$ & $55 (16\%)$ & $53.05\%  (27.5\%)$ \\
Random & $54.7\% (8\%)$ & $54.6 (18\%)$ & $54.6\% (33.2\%)$\\
\textsc{GradM} & $29.9\% (8.2\%)$ & $29.1\% (14.7\%)$ & $32.75\%  (23.2\%)$\\
\glister & $21.1\%  (8.6\%)$ & $17.2\% (16\%)$ & $14.4\% (22.2\%)$
\end{tabular}
}
\end{small}
\vspace{-4mm}
\end{table}

\begin{table}[b]
\vspace{-4mm}
\caption{Training ResNet18 with $S$=1\%  subsets every $R$=1 epoch from CIFAR10 using batch size $b$= 512, 256, 128. 
\alg can leverage larger mini-bath size and obtain a larger
accuracy gap to \craig and Random. For $b$=512, we have\! 1\! mini-batch (GD). 
Std is reported in Appendix Table\! \ref{table:batch_pm}.
}\label{table:batch}
\vspace{-2mm}
\begin{small}
\resizebox{\columnwidth}{!}{
\begin{tabular}{l|lllll}
\hline
 & \textsc{AdaC.} & \craig & Rand & \begin{tabular}[c]{@{}l@{}}Gap/\\ \craig\end{tabular} & \begin{tabular}[c]{@{}l@{}}Gap/\\ Rand\end{tabular} \\ \hline
\begin{tabular}[c]{@{}l@{}}GD~~ b=512\end{tabular}  & \textbf{58.32}\%  & 56.32\%  & 49.14\% & 1.69\% & \textbf{8.91}\% \\
\begin{tabular}[c]{@{}l@{}}SGD b=256\end{tabular} & \textbf{68.23}\%  & 58.3\%  & 60.7\%   & \textbf{9.93}\% & 8.16\% \\
\begin{tabular}[c]{@{}l@{}}SGD b=128\end{tabular} & \textbf{66.89}\% & 58.17\% & 65.46\% & \textbf{8.81}\% & 1.52\%
\end{tabular}
}
\end{small}
\end{table}

\begin{table*}[t] 
\caption{Test accuracy and percent of full data selected (in parentheses), when selecting $S$\% coresets every $R$ epochs from CIFAR-10 and Imbalanced CIFAR-10 to train ResNet18. 
}
\label{table:cifar_imb_ada_craig}
\vspace{-3mm}
\begin{small}
\begin{tabular}{lllll}
\hline
 & \text{\begin{tabular}[c]{@{}l@{}}ResNet20, ~CIFAR10\\  $S$ = $30\%$, ~~$R$ = 20\end{tabular}} & \text{\begin{tabular}[c]{@{}l@{}}ResNet20, ~CIFAR10\\  $S$ = $10\%$, ~~$R$ = 20\end{tabular}} & \text{\begin{tabular}[c]{@{}l@{}}ResNet18, ~CIFAR10-IMB\\  $S$ = $30\%$, ~~$R$ = 20\end{tabular}} & \text{\begin{tabular}[c]{@{}l@{}}ResNet18, ~CIFAR10-IMB\\  $S$ = $10\%$,~~ $R$ = 20\end{tabular}} \\ \hline
\alg & $\textbf{80.57\%} \pm 0.11$ \; $(\textbf{74.6\%})$ & $\textbf{70.6}\%  \pm 0.33$ \; $(\textbf{44.8\%})$ & $\textbf{85.7\%} \pm 0.1$ \; $(\textbf{74\%})$ & $\textbf{76\%} \pm 0.3$ \; $(\textbf{43.8\%})$\\
\craig & $65.8\% \pm 0.41$ \; $(90.9\%)$ & $58.5\% \pm  1.27$ \; $(60.75\%)$ &  $79.3\% \pm 1.6$ \; $(84.5\%)$& $71.6\% \pm 0.15$ \; $(56.4\%)$
\end{tabular}%
% }
\end{small}
\vspace{-3mm}
\end{table*}

\textbf{Frequency and size of subsets selection} 
Table \ref{table:cifar_imb_sgd_mom}, \ref{table:cifar_imb_ada_craig} shows the performance of different methods for selecting subsets of size $S\%$ of the data every $R$ epochs, from CIFAR-10 and imbalanced CIFAR-10. 
Table \ref{table:cifar_imb_sgd_mom} shows that selecting subsets of size 1\% every $R=5, 10, 20$ epochs with \alg achieves a superior performance compared to the baselines.
Table \ref{table:cifar_imb_ada_craig} shows that \alg can successfully select larger subsets of size $S=10\%, 30\%$ and outperform \craig (Std is reported in Appendix, Table \ref{table:class_imb_sgd_mom_full}).

\textbf{\alg speeds up training}
Fig \ref{fig:speedup} compares the speedup of various methods during training ResNet18 on 10\% subsets selected every $R=20$ epochs from BDD 100k and CIFAR-100. All the methods are trained to achieve a test accuracy between 72\% and 74\% on BDD 100k, and between 57\% and 50\% on CIFAR-100. 
On BDD 100k, \alg achieves 74\% accuracy in 100 epochs and training on full data achieves a similar performance in 45 epochs. For CIFAR-100, \alg achieves 59\% accuracy in 200 epochs and training on full data achieves a similar performance in 40 epochs.
Complete results on speedup and test accuracy of each method can be found in Appendix \ref{exp:bdd100k}, \ref{exp:cifar100}. We see that \alg achieves 2.5x speedup over training on full data and 1.7x over that of training on random subsets on BDD 100k. For CIFAR-100, \alg achieves 4.2x speedup over training on random subsets and 2.9x over training on full data. Compared to the baselines, \alg can achieve achieve the desired accuracy much faster.

\textbf{Effect of batch size}
Table \ref{table:batch} compares the performance of training with different batch sizes on subsets found by various methods. We see that training with larger batch size on subsets selected by \alg can achieve a superior accuracy. 
As \alg selects more diverse subsets with smaller weights, one can train with larger mini-batches on the subsets without increasing the gradient estimate error. In contrast, \craig subsets have elements with larger weights and hence training with fewer larger mini-batches has larger gradient error and does not improve the performance.

In summary, see that \alg consistently outperforms the baselines over various architectures,  optimizers, subset sizes, selection frequency, and batch sizes. 

\textbf{\alg selects more diverse subsets} Fig. \ref{subfig:R18_hist} shows the number of times different methods selected a particular elements during the entire training. We see that \alg successfully selects a more diverse set of examples compared to \craig. We note that \gradmatch may not be able to select subsets with the desired size, and instead augments the selected subset with randomly selected examples. Hence, it has a normal-shaped distribution. Fig. \ref{subfig:R18_forget} shows mean forgetting score for all examples within a class ranked by \alg at the end of training, over sliding window of size 100. We see that \alg prioritizes selecting less forgettable examples. This shows that indeed \alg is able to distinguish different groups of easier examples better, and hence can prevent catastrophic forgetting by including their representatives in the coresets. \looseness=-1
\begin{figure}[t]
    \centering
	\vspace{-3mm}
    \subfloat[BDD100k, ResNet50 \label{subfig:bdd}]{
	\includegraphics[width=.24\textwidth,height=.15\textwidth,trim=0 0 0 10mm 0]{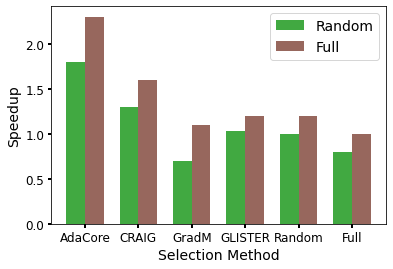}
    }~
    \subfloat[CIFAR100, ResNet18]{
    \includegraphics[width=.21\textwidth,height=.15\textwidth,trim=10mm 0 0 10mm 0]{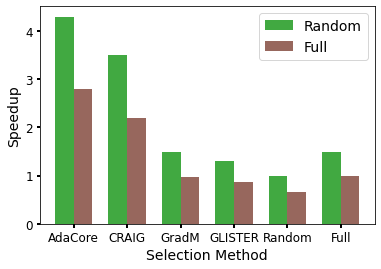}
    }
    \vspace{-2mm}
    \caption{ Speedup of various methods over training on random subsets and full data, for training ResNet18 on CIFAR100 and ResNet50 on BDD100k with batch size=128. 
    }
    \vspace{-5mm}
    \label{fig:speedup}
\end{figure}

\begin{figure}[t]
    \centering
    \vspace{-2mm}
    \subfloat[Forgetting vs class ranking \label{subfig:forget}]{
    \includegraphics[width=.21\textwidth,trim=10mm 0 10mm 0]{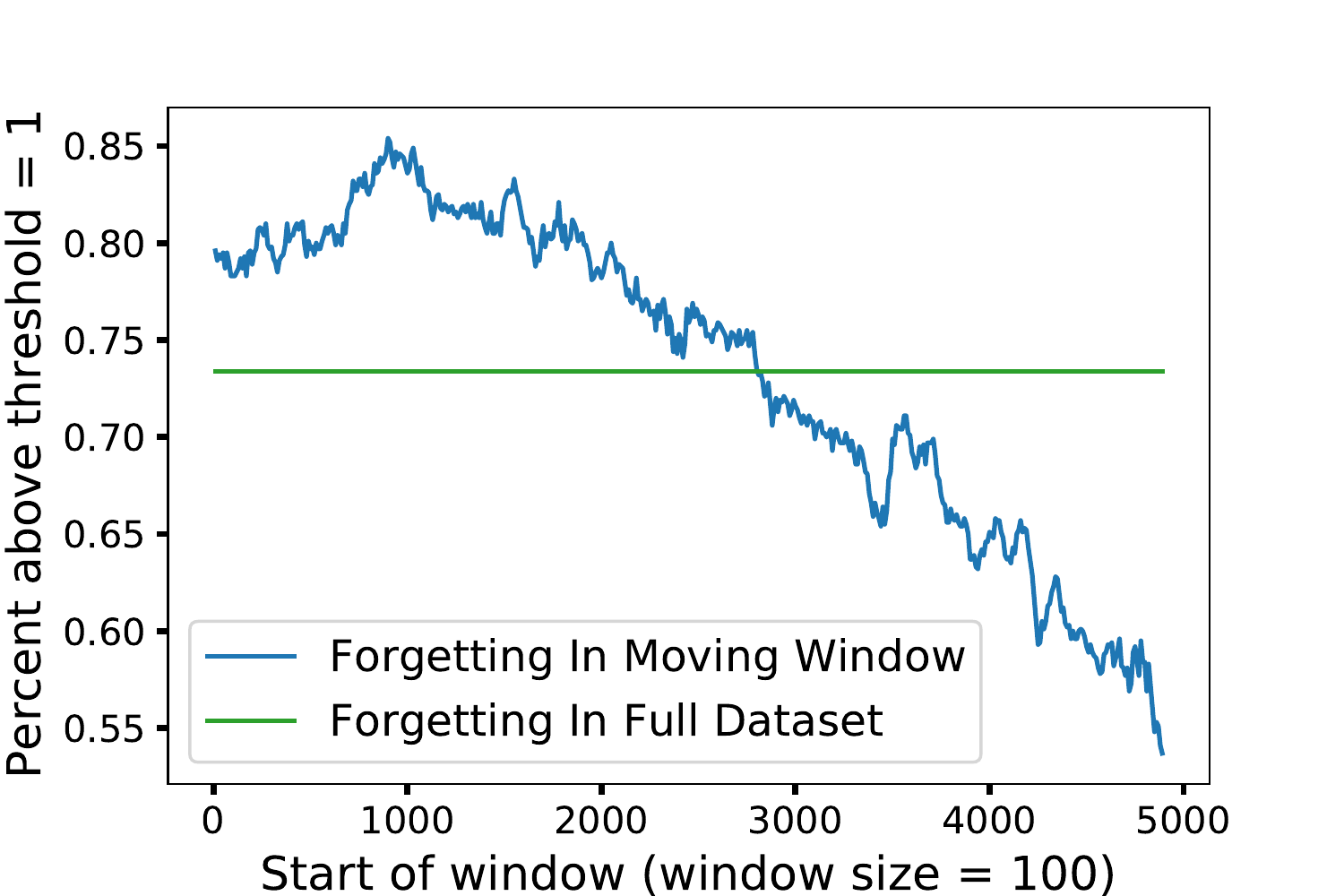}
    }\hfill 
    \subfloat[\!Uncertainty \!vs class ranking \label{subfig:certain}]{
    \includegraphics[width=.21\textwidth,trim=10mm 0 10mm 0]{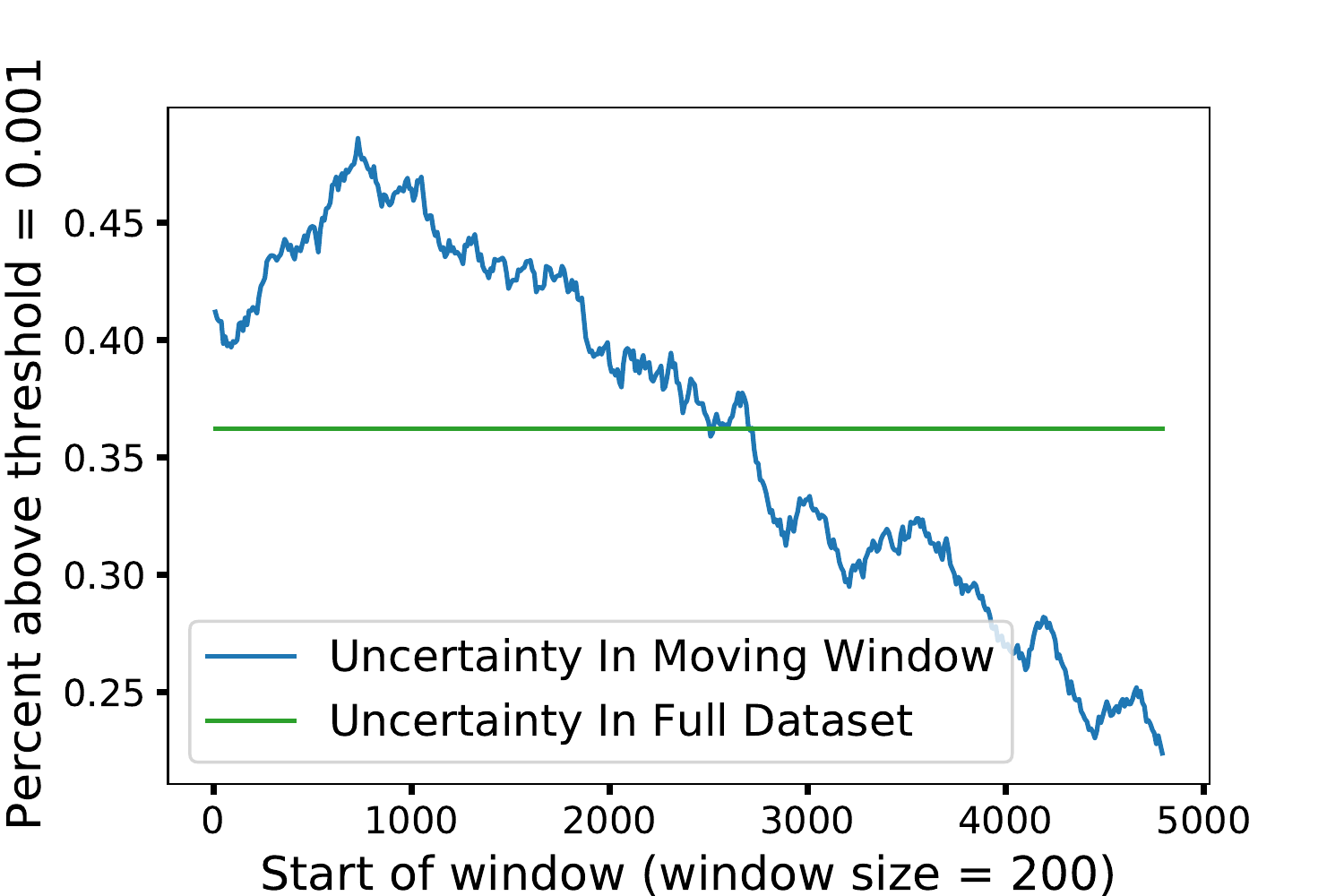}
    }
    \vspace{-4mm}
    \subfloat[Most selected \label{subfig:selected}]{\makebox[25mm][c]{
	\includegraphics[width=.07\textwidth,trim=10mm -12mm 10mm 0]{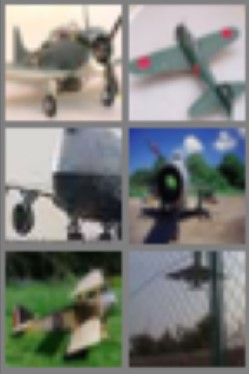}}
    }\hspace{20mm}
    \subfloat[Not selected \label{subfig:not_selected}]{\makebox[25mm][c]{
    \includegraphics[width=.07\textwidth,trim=10mm -12mm 10mm 0]{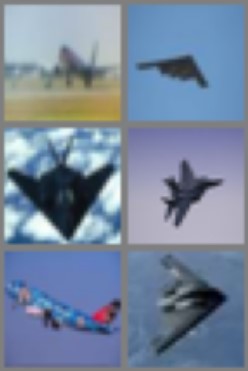}}
    }\vspace{-1mm}
    \caption{
    Training ResNet20 on $S$=1\% subsets of CIFAR-10 selected by \alg.
    (a) Forgetting scores, and (b) uncertainty of examples in a class sorted by \alg at the end of training. \alg prioritize selecting more forgettable and uncertain examples.
    (c) Six images selected by \alg most frequently (25 times) from the airplane class. (d) Subset of images never selected by \alg. %in CIFAR10 with ResNet20.
    }
    \label{fig:compare_reject_accept}
    \vspace{-5mm}
\end{figure}

\textbf{\alg vs Forgettability and Uncertainty}
Fig.
\ref{subfig:forget}, \ref{subfig:certain}
show
mean forgettability and uncertainty in sliding windows of size 100, 200 over examples sorted by \alg at the end of training.
We see
that \alg heavily biases its selections towards forgettable and uncertain points, as training proceeds. 
Interestingly, \ref{subfig:forget} reveals that \alg avoids the most forgettable samples in favor of slightly more memorable ones, suggesting that \alg can better distinguish easier groups of examples. Figure \ref{subfig:certain} shows similar bias towards uncertain samples. 
Fig. \ref{subfig:selected}, \ref{subfig:not_selected} show the most and least selected images by \alg, respectively. We see the redundancies in the never selected images, whereas images frequented by \alg are quite diverse in color, angles, occluded subjects, and airplane models. This confirms the effectiveness of \alg in extracting the most crucial subsets for learning and eliminating redundancies.

\vspace{-2mm}\section{Conclusion}
We proposed \alg, a method that leverages the topology of the dataset to extract salient subsets of large datasets for efficient machine learning. The key idea behind \alg is to dynamically incorporate the curvature { and gradient} of the loss function via an adaptive estimate of the Hessian to select weighted subsets (coresets) which closely approximate the preconditioned gradient of the full dataset. 
We proved exponential convergence rate for first and second-order optimization methods applied to \alg coresets, under certain assumptions. 
Our extensive experiments, using various optimizers e.g., SGD, AdaHessian, and Newton's method, show that \alg can extract higher quality coresets compared to baselines, rejecting potentially redundant data points. This speeds up the training of various machine learning models, such as logistic regression and neural networks, by over 4.5x while selecting fewer but more diverse data points for training.

% % Acknowledgements should only appear in the accepted version.
\section*{Acknowledgements}
This research was supported in part
by UCLA-Amazon Science Hub for Humanity and Artificial Intelligence.


\bibliography{example_paper}

\begin{thebibliography}{43}
\providecommand{\natexlab}[1]{#1}
\providecommand{\url}[1]{\texttt{#1}}
\expandafter\ifx\csname urlstyle\endcsname\relax
  \providecommand{\doi}[1]{doi: #1}\else
  \providecommand{\doi}{doi: \begingroup \urlstyle{rm}\Url}\fi

\bibitem[Alain et~al.(2015)Alain, Lamb, Sankar, Courville, and
  Bengio]{alain2015variance}
Alain, G., Lamb, A., Sankar, C., Courville, A., and Bengio, Y.
\newblock Variance reduction in sgd by distributed importance sampling.
\newblock \emph{arXiv preprint arXiv:1511.06481}, 2015.

\bibitem[Allen-Zhu et~al.(2016)Allen-Zhu, Yuan, and
  Sridharan]{allen2016exploiting}
Allen-Zhu, Z., Yuan, Y., and Sridharan, K.
\newblock Exploiting the structure: Stochastic gradient methods using raw
  clusters.
\newblock In \emph{Advances in Neural Information Processing Systems}, pp.\
  1642--1650, 2016.

\bibitem[Asi \& Duchi(2019)Asi and Duchi]{asi2019importance}
Asi, H. and Duchi, J.~C.
\newblock The importance of better models in stochastic optimization.
\newblock \emph{arXiv preprint arXiv:1903.08619}, 2019.

\bibitem[Bekas et~al.(2007)Bekas, Kokiopoulou, and Saad]{BEKAS20071214}
Bekas, C., Kokiopoulou, E., and Saad, Y.
\newblock An estimator for the diagonal of a matrix.
\newblock \emph{Applied Numerical Mathematics}, 57\penalty0 (11):\penalty0
  1214--1229, 2007.
\newblock ISSN 0168-9274.
\newblock \doi{https://doi.org/10.1016/j.apnum.2007.01.003}.
\newblock URL
  \url{https://www.sciencedirect.com/science/article/pii/S0168927407000244}.
\newblock Numerical Algorithms, Parallelism and Applications (2).

\bibitem[Bertsekas(1982)]{bertsekas1982projected}
Bertsekas, D.~P.
\newblock Projected newton methods for optimization problems with simple
  constraints.
\newblock \emph{SIAM Journal on control and Optimization}, 20\penalty0
  (2):\penalty0 221--246, 1982.

\bibitem[Birodkar et~al.(2019)Birodkar, Mobahi, and
  Bengio]{birodkar2019semantic}
Birodkar, V., Mobahi, H., and Bengio, S.
\newblock Semantic redundancies in image-classification datasets: The $10\%$
  you don't need.
\newblock \emph{arXiv preprint arXiv:1901.11409}, 2019.

\bibitem[Boyd \& Vandenberghe(2004)Boyd and Vandenberghe]{10.5555/993483}
Boyd, S. and Vandenberghe, L.
\newblock \emph{Convex Optimization}.
\newblock Cambridge University Press, USA, 2004.
\newblock ISBN 0521833787.

\bibitem[Chen et~al.(2001)Chen, Donoho, and Saunders]{chen2001atomic}
Chen, S.~S., Donoho, D.~L., and Saunders, M.~A.
\newblock Atomic decomposition by basis pursuit.
\newblock \emph{SIAM review}, 43\penalty0 (1):\penalty0 129--159, 2001.

\bibitem[Coleman et~al.(2020)Coleman, Yeh, Mussmann, Mirzasoleiman, Bailis,
  Liang, Leskovec, and Zaharia]{coleman2020selection}
Coleman, C., Yeh, C., Mussmann, S., Mirzasoleiman, B., Bailis, P., Liang, P.,
  Leskovec, J., and Zaharia, M.
\newblock Selection via proxy: Efficient data selection for deep learning.
\newblock In \emph{International Conference on Learning Representations
  (ICLR)}, 2020.

\bibitem[Deng(2012)]{deng2012mnist}
Deng, L.
\newblock The mnist database of handwritten digit images for machine learning
  research.
\newblock \emph{IEEE Signal Processing Magazine}, 29\penalty0 (6):\penalty0
  141--142, 2012.

\bibitem[Donoho(2006)]{donoho2006compressed}
Donoho, D.~L.
\newblock Compressed sensing.
\newblock \emph{IEEE Transactions on information theory}, 52\penalty0
  (4):\penalty0 1289--1306, 2006.

\bibitem[Elenberg et~al.(2018)Elenberg, Khanna, Dimakis, Negahban,
  et~al.]{elenberg2018restricted}
Elenberg, E.~R., Khanna, R., Dimakis, A.~G., Negahban, S., et~al.
\newblock Restricted strong convexity implies weak submodularity.
\newblock \emph{Annals of Statistics}, 46\penalty0 (6B):\penalty0 3539--3568,
  2018.

\bibitem[Ghorbani \& Zou(2019)Ghorbani and Zou]{ghorbani2019data}
Ghorbani, A. and Zou, J.
\newblock Data shapley: Equitable valuation of data for machine learning.
\newblock In \emph{International Conference on Machine Learning}, pp.\
  2242--2251. PMLR, 2019.

\bibitem[He et~al.(2016)He, Zhang, Ren, and Sun]{he2016deep}
He, K., Zhang, X., Ren, S., and Sun, J.
\newblock Deep residual learning for image recognition.
\newblock In \emph{Proceedings of the IEEE conference on computer vision and
  pattern recognition}, pp.\  770--778, 2016.

\bibitem[Hofmann et~al.(2015)Hofmann, Lucchi, Lacoste-Julien, and
  McWilliams]{hofmann2015variance}
Hofmann, T., Lucchi, A., Lacoste-Julien, S., and McWilliams, B.
\newblock Variance reduced stochastic gradient descent with neighbors.
\newblock In \emph{Advances in Neural Information Processing Systems}, pp.\
  2305--2313, 2015.

\bibitem[Katharopoulos \& Fleuret(2018)Katharopoulos and
  Fleuret]{katharopoulos2018not}
Katharopoulos, A. and Fleuret, F.
\newblock Not all samples are created equal: Deep learning with importance
  sampling.
\newblock In \emph{International conference on machine learning}, pp.\
  2525--2534. PMLR, 2018.

\bibitem[Killamsetty et~al.(2020)Killamsetty, Sivasubramanian, Ramakrishnan,
  and Iyer]{killamsetty2020glister}
Killamsetty, K., Sivasubramanian, D., Ramakrishnan, G., and Iyer, R.
\newblock Glister: Generalization based data subset selection for efficient and
  robust learning.
\newblock \emph{arXiv preprint arXiv:2012.10630}, 2020.

\bibitem[Killamsetty et~al.(2021)Killamsetty, Sivasubramanian, Mirzasoleiman,
  Ramakrishnan, De, and Iyer]{killamsetty2021grad}
Killamsetty, K., Sivasubramanian, D., Mirzasoleiman, B., Ramakrishnan, G., De,
  A., and Iyer, R.
\newblock Grad-match: A gradient matching based data subset selection for
  efficient learning.
\newblock \emph{arXiv preprint arXiv:2103.00123}, 2021.

\bibitem[Krizhevsky et~al.(2009)Krizhevsky, Nair, and Hinton]{cifar10}
Krizhevsky, A., Nair, V., and Hinton, G.
\newblock Cifar-10 (canadian institute for advanced research).
\newblock 2009.
\newblock URL \url{http://www.cs.toronto.edu/~kriz/cifar.html}.

\bibitem[Kyrillidis et~al.(2013)Kyrillidis, Becker, Cevher, and
  Koch]{kyrillidis2013sparse}
Kyrillidis, A., Becker, S., Cevher, V., and Koch, C.
\newblock Sparse projections onto the simplex.
\newblock In \emph{International Conference on Machine Learning}, pp.\
  235--243. PMLR, 2013.

\bibitem[Liu et~al.(2020)Liu, Zhu, and Belkin]{liu2020toward}
Liu, C., Zhu, L., and Belkin, M.
\newblock Toward a theory of optimization for over-parameterized systems of
  non-linear equations: the lessons of deep learning.
\newblock \emph{arXiv preprint arXiv:2003.00307}, 2020.

\bibitem[Loshchilov \& Hutter(2015)Loshchilov and Hutter]{loshchilov2015online}
Loshchilov, I. and Hutter, F.
\newblock Online batch selection for faster training of neural networks.
\newblock \emph{arXiv preprint arXiv:1511.06343}, 2015.

\bibitem[Martens \& Grosse(2015)Martens and Grosse]{martens2015optimizing}
Martens, J. and Grosse, R.
\newblock Optimizing neural networks with kronecker-factored approximate
  curvature.
\newblock In \emph{International conference on machine learning}, pp.\
  2408--2417. PMLR, 2015.

\bibitem[Minoux(1978)]{minoux1978accelerated}
Minoux, M.
\newblock Accelerated greedy algorithms for maximizing submodular set
  functions.
\newblock In \emph{Optimization techniques}, pp.\  234--243. Springer, 1978.

\bibitem[Mirzasoleiman et~al.(2013)Mirzasoleiman, Karbasi, Sarkar, and
  Krause]{mirzasoleiman2013distributed}
Mirzasoleiman, B., Karbasi, A., Sarkar, R., and Krause, A.
\newblock Distributed submodular maximization: Identifying representative
  elements in massive data.
\newblock In \emph{Advances in Neural Information Processing Systems}, pp.\
  2049--2057, 2013.

\bibitem[Mirzasoleiman et~al.(2015)Mirzasoleiman, Badanidiyuru, Karbasi,
  Vondr{\'a}k, and Krause]{mirzasoleiman2015lazier}
Mirzasoleiman, B., Badanidiyuru, A., Karbasi, A., Vondr{\'a}k, J., and Krause,
  A.
\newblock Lazier than lazy greedy.
\newblock In \emph{Twenty-Ninth AAAI Conference on Artificial Intelligence},
  2015.

\bibitem[Mirzasoleiman et~al.(2020)Mirzasoleiman, Bilmes, and
  Leskovec]{mirzasoleiman2020coresets}
Mirzasoleiman, B., Bilmes, J., and Leskovec, J.
\newblock Coresets for data-efficient training of machine learning models.
\newblock In \emph{International Conference on Machine Learning}, pp.\
  6950--6960. PMLR, 2020.

\bibitem[Natarajan(1995)]{natarajan1995sparse}
Natarajan, B.~K.
\newblock Sparse approximate solutions to linear systems.
\newblock \emph{SIAM journal on computing}, 24\penalty0 (2):\penalty0 227--234,
  1995.

\bibitem[Nocedal(1980)]{Nocedal}
Nocedal, J.
\newblock Updating quasi-newton matrices with limited storage.
\newblock \emph{Mathematics of Computation}, 35\penalty0 (151):\penalty0
  773--782, 1980.
\newblock ISSN 00255718, 10886842.
\newblock URL \url{http://www.jstor.org/stable/2006193}.

\bibitem[Pilanci et~al.(2012)Pilanci, El~Ghaoui, and
  Chandrasekaran]{pilanci2012recovery}
Pilanci, M., El~Ghaoui, L., and Chandrasekaran, V.
\newblock Recovery of sparse probability measures via convex programming.
\newblock 2012.

\bibitem[Qian(1999)]{qian1999momentum}
Qian, N.
\newblock On the momentum term in gradient descent learning algorithms.
\newblock \emph{Neural networks}, 12\penalty0 (1):\penalty0 145--151, 1999.

\bibitem[Robbins \& Monro(1951)Robbins and Monro]{robbins1951stochastic}
Robbins, H. and Monro, S.
\newblock A stochastic approximation method.
\newblock \emph{The annals of mathematical statistics}, pp.\  400--407, 1951.

\bibitem[Schaul et~al.(2013)Schaul, Zhang, and LeCun]{schaul2013no}
Schaul, T., Zhang, S., and LeCun, Y.
\newblock No more pesky learning rates.
\newblock In \emph{International Conference on Machine Learning}, pp.\
  343--351. PMLR, 2013.

\bibitem[Schaul et~al.(2015)Schaul, Quan, Antonoglou, and
  Silver]{schaul2015prioritized}
Schaul, T., Quan, J., Antonoglou, I., and Silver, D.
\newblock Prioritized experience replay.
\newblock \emph{arXiv preprint arXiv:1511.05952}, 2015.

\bibitem[Schwartz et~al.(2019)Schwartz, Dodge, Smith, and
  Etzioni]{schwartz2019green}
Schwartz, R., Dodge, J., Smith, N.~A., and Etzioni, O.
\newblock Green ai.
\newblock \emph{arXiv preprint arXiv:1907.10597}, 2019.

\bibitem[Strubell et~al.(2019)Strubell, Ganesh, and
  McCallum]{strubell2019energy}
Strubell, E., Ganesh, A., and McCallum, A.
\newblock Energy and policy considerations for deep learning in nlp.
\newblock In \emph{Proceedings of the 57th Annual Meeting of the Association
  for Computational Linguistics}, pp.\  3645--3650, 2019.

\bibitem[Tibshirani(1996)]{tibshirani1996regression}
Tibshirani, R.
\newblock Regression shrinkage and selection via the lasso.
\newblock \emph{Journal of the Royal Statistical Society: Series B
  (Methodological)}, 58\penalty0 (1):\penalty0 267--288, 1996.

\bibitem[Toneva et~al.(2018)Toneva, Sordoni, des Combes, Trischler, Bengio, and
  Gordon]{toneva2018empirical}
Toneva, M., Sordoni, A., des Combes, R.~T., Trischler, A., Bengio, Y., and
  Gordon, G.~J.
\newblock An empirical study of example forgetting during deep neural network
  learning.
\newblock In \emph{International Conference on Learning Representations}, 2018.

\bibitem[Wolsey(1982)]{wolsey1982analysis}
Wolsey, L.~A.
\newblock An analysis of the greedy algorithm for the submodular set covering
  problem.
\newblock \emph{Combinatorica}, 2\penalty0 (4):\penalty0 385--393, 1982.

\bibitem[Xu et~al.(2020)Xu, Roosta, and Mahoney]{xu2020second}
Xu, P., Roosta, F., and Mahoney, M.~W.
\newblock Second-order optimization for non-convex machine learning: An
  empirical study.
\newblock In \emph{Proceedings of the 2020 SIAM International Conference on
  Data Mining}, pp.\  199--207. SIAM, 2020.

\bibitem[Yao et~al.(2018)Yao, Xu, Roosta-Khorasani, and
  Mahoney]{yao2018inexact}
Yao, Z., Xu, P., Roosta-Khorasani, F., and Mahoney, M.~W.
\newblock Inexact non-convex newton-type methods.
\newblock \emph{arXiv preprint arXiv:1802.06925}, 2018.

\bibitem[Yao et~al.(2020)Yao, Gholami, Shen, Keutzer, and
  Mahoney]{yao2020adahessian}
Yao, Z., Gholami, A., Shen, S., Keutzer, K., and Mahoney, M.~W.
\newblock Adahessian: An adaptive second order optimizer for machine learning.
\newblock \emph{arXiv preprint arXiv:2006.00719}, 2020.

\bibitem[Yu et~al.(2020)Yu, Chen, Wang, Xian, Chen, Liu, Madhavan, and
  Darrell]{bdd100k}
Yu, F., Chen, H., Wang, X., Xian, W., Chen, Y., Liu, F., Madhavan, V., and
  Darrell, T.
\newblock Bdd100k: A diverse driving dataset for heterogeneous multitask
  learning.
\newblock In \emph{IEEE/CVF Conference on Computer Vision and Pattern
  Recognition (CVPR)}, June 2020.

\end{thebibliography}
\bibliographystyle{icml2022}

%%%%%%%%%%%%%%%%%%%%%%%%%%%%%%%%%%%%%%%%%%%%%%%%%%%%%%%%%%%%%%%%%%%%%%%%%%%%%%%
%%%%%%%%%%%%%%%%%%%%%%%%%%%%%%%%%%%%%%%%%%%%%%%%%%%%%%%%%%%%%%%%%%%%%%%%%%%%%%%
% APPENDIX
%%%%%%%%%%%%%%%%%%%%%%%%%%%%%%%%%%%%%%%%%%%%%%%%%%%%%%%%%%%%%%%%%%%%%%%%%%%%%%%
%%%%%%%%%%%%%%%%%%%%%%%%%%%%%%%%%%%%%%%%%%%%%%%%%%%%%%%%%%%%%%%%%%%%%%%%%%%%%%%
\newpage
\appendix
\onecolumn
\section{Proofs of Theorems}

\subsection{Proof of Theorem \ref{thm:newton}}

\newtonrestate*

\begin{proof}
\label{proof:4.1}
We prove Theorem \ref{thm:newton} (similarly to the proof of  Newton's method in \cite{10.5555/993483})  for the following general update rule for $0\leq k\leq1$:
\begin{align}
    \Delta w_t = \mathbf{H}_t^{-k}\mathbf{g}_t\\
    w_{t+1} = w_t - \eta\Delta w_t
\end{align}
    For $k=1$, this corresponds to the update rule of the Newton's method. 
    Define $\lambda(w_t) = (\mathbf{g}_t^T \mathbf{H}_t^{-k} \mathbf{g}_t)^{1/2}$. 
    Since $\Lc (w)$ is $\beta$-smooth, we have
    \begin{align}
        \Lc (w_{t+1}) &\leq \Lc (w_t)-\eta \mathbf{g}_t^T \Delta w_t + \frac{\eta^2 \beta \|\Delta w_t\|^2}{2}\\
        &\leq \Lc (w_t)-\eta\lambda(w_t)^2+\frac{\beta}{2\alpha^k}\eta^2\lambda(w_t)^2,
    \end{align}
    where in the last equality, we used
    \begin{align}
        \lambda(w_t)=\Delta w_t \mathbf{H}_t^k \Delta w_t^T.
    \end{align}
    Therefore, using step size $\hat{\eta}=\frac{\alpha^k}{\beta}$ we have $w_{t+1} = w_t-\hat{\eta}\Delta w_t$
    \begin{align}
        \Lc (w_{t+1})\leq \Lc (w_t)-\frac{1}{2}\hat{\eta}\lambda(w_t)^2
    \end{align}
    Since $\alpha I \preceq \mathbf{H}_t \preceq \beta \emph{I}$, we have
    \begin{align}
        \lambda(w_t)^2 = \mathbf{g}_t^T \mathbf{H}_t^{-k} \mathbf{g}_t \geq \frac{1}{\beta^k}\|\mathbf{g}_t\|^2,
    \end{align}
    and therefore $\Lc$ decreases as follows,
    \begin{align}
        \Lc (w_{t+1}) - \Lc (w_t) \leq -\frac{1}{2\beta^k}\hat{\eta}\|\mathbf{g}_t\|^2 =-\frac{\alpha^k}{2\beta^{k+1}}\|\mathbf{g}_t\|^2.
    \end{align}
    Now for the subset, from Eq. \eqref{eq:main} we have that $\|\mathbf{H}^{-1}_t \textbf{g}_t-\sum_{j\in S}\gamma_{t,j}\mathbf{H}_{t,j}^{-1} \textbf{g}_{t,j}\|\leq\epsilon$. Hence, via reverse triangle inequality $\|\mathbf{H}^{-1}_t \textbf{g}_t\|\leq \|\sum_{j\in S}\gamma_{t,j}\mathbf{H}_{t,j}^{-1} \textbf{g}_{t,j}\|+\epsilon$, and we get
    \begin{align}
        \frac{\|\mathbf{g}_t\|}{\beta}\leq\| (\mathbf{H}_t)^{-1}\mathbf{g}_t\|\leq \| (\mathbf{H}^S_t)^{-1}\mathbf{g}^S_t\pmb\gamma\|+\epsilon \leq \frac{\|\mathbf{g}^S_t\|}{\alpha}+\epsilon,
        \label{eq:ulb_subset}
    \end{align}
    where $\mathbf{g}^S_t = \sum_{j \in S} \mathbf{g}_{t,j}$ and $\mathbf{H}^S_t=\sum_{j\in S}\mathbf{H}_{t,j.}$ are the gradient and Hessian of the subset respectively. In Eq. \eqref{eq:ulb_subset} the RHS follows from operator norms and the LHS follows from 
    the following lower bound on the norm of the product of two matrices:
    \begin{align}
        \begin{aligned}
        \|AB\| &= \max_{\|x\|=1} \|x^TAB\|\\
        &= \max_{\|x\|=1} \|x^TA\|\left\|\frac{x^TA}{\|x^TA\|}B\right\|\\
        &\ge \max_{\|x\|=1} \sigma_{\min(A)}\left\|\frac{x^TA}{\|x^TA\|}B\right\|\\
        &= \max_{\|y\|=1} \sigma_{\min(A)}\left\|y^TB\right\|\\
        &= \sigma_{\min(A)}\|B\|,
        \end{aligned}
        \label{stack}
    \end{align}
    Hence,
    \begin{align}
        \|\mathbf{g}_t^S\|\geq\frac{\alpha}{\beta} (\|\mathbf{g}_t\|-\beta\epsilon)
    \end{align}
    Therefore, on the subset we have
    \begin{align}
        \Lc (w_{t+1}) - \Lc (w_t) 
        &\leq -\frac{\alpha^k}{2\beta^{k+1}}\|\mathbf{g}^S_t\|^2\\
        &\leq  
        -\frac{\alpha^{k}}{2\beta^{k+1}} (\frac{\alpha}{\beta})^2(\|\mathbf{g}_t\|-\beta\epsilon)^2\\
        &= -\frac{\alpha^{k+2}}{2\beta^{k+3}} (\|\mathbf{g}_t\|-\beta\epsilon)^2.
    \end{align}
The algorithm stops descending when $\|\mathbf{g}_t\|=\beta\epsilon$. From strong convexity we know that    
\begin{align}
    \|\mathbf{g}_t\|=\beta\epsilon\geq \alpha\|w-w_*\|
\end{align}
Hence, we get
\begin{align}
    \|w-w_*\|\leq \beta\epsilon/\alpha.
\end{align}
    
    As such we have Corollary \ref{thm:adahessian}. and when we set $k=1$ we have proof of Theorem \ref{thm:newton}.
\end{proof}

\textbf{Descent property for Equation \ref{eq:main}}
\label{proof:diagconv}
For a strongly convex function, $\Lc$,  we have that the diagonal elements of the Hessian 
are positive \cite{yao2020adahessian}.  This allows the diagonal to approximate the full Hessian which has good convergence properties. 

Given a function $\Lc (w)$ which is strongly convex and strictly smooth, %in $\mathbb{R}^d$, 
we have that $\nabla^2 \Lc (w)$ is upper and lower bounded by two constants $\beta$ and $\alpha$, so that $\alpha I \leq \nabla^2 \Lc (w) \leq \beta I$, for all $w$. For a strongly convex function the diagonal elements in diag($\mathbf{H}_{t}$) are all positive, and we have:
\begin{align}
    \alpha \leq e_i^T \mathbf{H}_{t} e_i = e_i^T \text{diag}(\mathbf{H}_{t}) e_i = \text{diag}(\mathbf{H}_{t})_{i,j} \leq \beta
\end{align}
where $e_j$ represents the natural basis vectors. Therefore, the diagonal entries of $\text{diag}(\mathbf{H}_{t})$ are in the range $[\alpha,\beta]$. Therefore, the average of a subset of the numbers are still in the range $[\alpha,\beta]$.  As such, we can prove that 
Eq. (\ref{eq:H_sub}) has the same convergence rate as its full matrix counterpart, by following the same proof as Theorem \ref{thm:newton}.

\subsection{Proof of Theorem \ref{thm:pl} and \ref{thm:pl-sgd}}
\label{proof:4.3}
% \textbf{Polyak-Lojasiewicz (PL) condition}

A loss function $\mathcal{L}(w)$ is considered $\mu$-PL on a set $\mathcal{S}$, if the following holds:
\begin{align}
    \frac{1}{2}\|\mathbf{g}\|^{2} \geq \mu\left(\mathcal{L}(w)-\mathcal{L}\left(w_{*}\right)\right), \forall w \in \mathcal{S}
    \label{eq:pl}
\end{align}
where $w_{*}$ is a global minimizer.
When additionally $\mathcal{L}\left(w_{*}\right) = 0$, 
the $\mu$-$\text{PL}$ condition is equivalent to the $\mu$-$\text{PL}^{*}$ condition 
\begin{align}
    \frac{1}{2}\|\mathbf{g}\|^{2} \geq \mu\mathcal{L}(w), \forall w \in \mathcal{S}.
\end{align}

\plrestate*

For Lipschitz continuous $\mathbf{g}$ and $\mu$-PL$^*$ condition, gradient descent on the entire dataset yields
\begin{align}
    \Lc(w_{t+1}) - \Lc(w_{t}) \leq -\frac{\eta} {2}\|\mathbf{g}_{t}\|^2 \leq -\eta\mu \Lc(w_{t}),
\end{align}
% [Note: sign?]
and,

\begin{align}
    \Lc(w_{t})\leq(1-\eta\mu)^t \Lc(w_0)%\label{eq:grad_rate},
\end{align}
which was shown in \cite{liu2020toward}.

We build upon this result to an \alg subset. 

\begin{proof}    
    From Eq. \eqref{eq:ulb_subset} we have that
    \begin{align}
        \frac{\|\mathbf{g}_t\|}{\beta}\leq\| (\mathbf{H}_t)^{-1}\mathbf{g}_t\|\leq \| (\mathbf{H}^S_t)^{-1}\mathbf{g}_t^S\pmb\gamma\|+\epsilon \leq \frac{\|\mathbf{g}_t^S\|}{\alpha}+\epsilon
    \end{align}
    Hence solving for $\|\mathbf{g}_t^S\|$ we have,
    \begin{equation}
        \|\mathbf{g}_t^S\|\geq \frac{\alpha}{\beta}(\|\mathbf{g}_t\|-\beta\epsilon).
        \label{eq:s_error}
    \end{equation}

    For the subset we have
    \begin{align}
        \Lc(w_{t+1}) - \Lc(w_t) 
        &\leq -\frac{\eta}{2}\|\mathbf{g}_t^S\|^2
    \end{align}    
    By substituting Eq. (\ref{eq:s_error}) we have.
    \begin{align}
        &\leq -\frac{\eta \alpha^2}{2 \beta^2}(\|\mathbf{g}_t\|-\beta\epsilon)^2\\
        &= -\frac{\eta\alpha^2} {2\beta^2}(\|\mathbf{g}_t\|^2+\beta^2\epsilon^2-2\beta\epsilon\|\mathbf{g}_t\|)\label{eq:pre_spectral_upper}\\
        &\leq - \frac{\eta\alpha^2} {2\beta^2}(\|\mathbf{g}_t\|^2+\beta^2\epsilon^2-2\beta\epsilon \nabla_{\max})\label{eq:pl_before_ada}\\
        &\leq -\frac{\eta\alpha^2} {2\beta^2}(2\mu \Lc(w_t)+\beta^2\epsilon^2-2\beta\epsilon \nabla_{\max}) \label{eq:pl_applied_adacore}
    \end{align}
    % Where we 
    Where we can upper bound the norm of $\mathbf{g}_t$ in Eq. (\ref{eq:pre_spectral_upper}) by a constant $\nabla_{max}$. And Eq. (\ref{eq:pl_applied_adacore}) follows from the $\mu$-PL condition from Eq. (\ref{eq:pl}).

    Hence,
    \begin{align}
        \Lc(w_{t+1}) \leq (1-\frac{\eta\mu\alpha^2} {\beta^2}) \Lc(w_t) - \frac{\eta\alpha^2} {2\beta^2}(\beta^2\epsilon^2-2\beta\epsilon \nabla_{\max})
    \end{align}
    Since, $\sum_{j=0}^k(1-\frac{\eta\mu\alpha^2} {\beta^2})^j\leq\frac{\beta^2}{\eta\mu\alpha^2}$, for a constant learning rate $\eta$ we get
    \begin{align}
        \Lc(w_{t+1}) \leq (1-\frac{\eta\mu\alpha^2} {\beta^2})^{t+1} \Lc(w_0) - \frac{\eta\alpha^2} {2\beta^2}(\beta^2\epsilon^2-2\beta\epsilon \nabla_{\max})
        \label{eq:adacore_convergence}
    \end{align}
   
    \end{proof}

\plsgdrestate*
For Lipschitz continuous $\mathbf{g}$ and $\mu$-PL$^*$ condition, gradient descent on the entire dataset yields
\begin{align}
    \Lc(w_{t+1}) - \Lc(w_t) \leq -\frac{\eta} {2}\|\mathbf{g}_t\|^2 \leq -\eta\mu \Lc(w_t),
\end{align}

and,

\begin{align}
    \Lc(w_{t})\leq(1-\eta\mu)^t \Lc(w_0)\label{eq:grad_rate},
\end{align}
which was shown in \cite{liu2020toward}.

We build upon this result to an \alg subset. 

\begin{proof}    
    From Eq. \eqref{eq:ulb_subset} we have that
    \begin{align}
        \frac{\|\mathbf{g}_t\|}{\beta}\leq\| (\mathbf{H}_t)^{-1}\mathbf{g}_t\|\leq \| (\mathbf{H}^S_t)^{-1}\mathbf{g}_t^S\pmb\gamma\|+\epsilon \leq \frac{\|\mathbf{g}_t^S\|}{\alpha}+\epsilon
    \end{align}
    Hence solving for $\|\mathbf{g}_t^S\|$ we have,
    \begin{equation}
        \|\mathbf{g}_t^S\|\geq \frac{\alpha}{\beta}(\|\mathbf{g}_t\|-\beta\epsilon).
        % \label{eq:s_error}
    \end{equation}

    For the subset we have
    \begin{align}
        \Lc(w_{t+1}) - \Lc(w_t) 
        &\leq -\frac{\eta}{2}\|\mathbf{g}_t^S\|^2
    \end{align}    

    Fixing $w_t$ and taking expectation with respect to the randomness in the choice of the batch $i_t^{(1)} \dots i_t^{(m)}$ (noting that those indices are i.i.d.), we have
    \begin{align}
        \mathbb{E}_{i_t^{(1)} \dots i_t^{(m)}} [\Lc(w_{t+1}) - \Lc(w_t)] 
        &\leq -\eta(\alpha - \eta \frac{\beta}{m}(\alpha \frac{m-1}{2}+\beta)\Lc(w_t)\\
        &\leq -\underbrace{\eta(1 - \frac{\eta \beta (m-1)}{2m})}_{c_1}\|\mathbf{g}_t\|^2 + \underbrace{\frac{\eta^2\beta \lambda}{m}}_{c_2}\Lc(w_t)\\
        &\leq -c_1 \frac{\alpha^2}{\beta^2}\|\mathbf{g}_t + \beta \epsilon\|)^2 + c_2\Lc(w_t)
        % \label{eq:pl_applied_adacore}
    \end{align}
    We can upper bound the norm of $\mathbf{g}_t$ in Eq. (\ref{eq:pl_applied_adacore}) by a constant $\nabla_{max}$. And Eq. (\ref{eq:pl_applied_adacore}) follows from the $\mu$-PL condition from Eq. (\ref{eq:pl}) and assuming $\eta \leq \frac{2}{\beta}$.
    \begin{align}
        &\leq -c_1\frac{\alpha^2}{\beta^2}(\mu\Lc(w_t) - 2 \nabla_{max}\beta \epsilon + \beta^2\epsilon^2) + c_2 \Lc(w_t)\\
        &\leq -\eta(1 - \frac{\eta \beta (m-1)}{2m})\frac{\alpha^2}{\beta^2}(\mu\Lc(w_t) - 2 \nabla_{max}\beta \epsilon + \beta^2\epsilon^2) + \frac{\eta^2\beta \lambda}{m} \Lc(w_t)\\
        &= \eta(\mu\frac{\alpha^2}{\beta^2} -\eta \frac{\beta}{m}(\mu \frac{\alpha^2(m-1)}{\beta^2 2}+\lambda))\Lc(w_t) + \frac{\eta^2\beta \lambda}{m} \Lc(w_t)+ c_1 \frac{\alpha^2}{\beta^2}(2\nabla_{max} \beta \epsilon - \beta^2\epsilon^2)\\
        &=\eta \mu \frac{\alpha^2}{\beta^2} (1 - \eta \beta \frac{m-1}{2m})\mathbb{E}[\Lc(w_t)] + \eta\frac{\alpha^2}{\beta^2}(1-\eta\beta\frac{m-1}{2m})(2\nabla_{max} \beta \epsilon - \beta^2\epsilon^2)
    \end{align}
    By optimizing the quadratic term in the upper bound with respect to $\eta$ we get $\eta = \frac{m}{\beta(m-1)}$. 
    \begin{align}
    \mathbb{E}[\Lc(w_{t+1})] \leq (1-\frac{\mu \alpha^2m}{2\beta^2(m-1)}) \mathbb{E}[\Lc(w_t)] + \frac{\alpha^2m}{\beta^2}\frac{2\nabla_{max} \beta \epsilon - \beta^2\epsilon^2}{2\beta(m-1)}
    \end{align}

    Hence,
    \begin{align}
        \mathbb{E}[\Lc(w_{t+1})] \leq \left(1-\frac{\eta^*(m)\mu \alpha^2}{2\beta}\right) \mathbb{E}[\Lc(w_t)] + \frac{\alpha^2\eta^*(m)}{\beta}(\nabla_{max}  \epsilon - \beta\epsilon^2/2)%\label{eq:adacore_convergence_sgd}
    \end{align}

    \end{proof}

\subsection{Discussion on Greedy 
to Extract Near-optimal Coresets
}\label{appx:alg}
As discussed in Section \ref{sec:alg}, a greedy algorithm can be applied to find near-optimal coresets that estimate the general descent direction with an error of at most $\epsilon$ by solving the submodular cover problem Eq. \eqref{eq:cover}.
For completeness, we include the pseudocode of the greedy algorithm in Algorithm \ref{alg:greedy}. The \alg algorithm is run per class.

\begin{algorithm}[ht]
	\begin{algorithmic}[1]
		\REQUIRE Set of component functions $f_i$ for $i \in V=[n]\}$.
		\ENSURE Subset $S \subseteq V$ with corresponding per-element stepsizes $\{\gamma\}_{j\in S}$.
		\STATE $S_0 \leftarrow \emptyset, s_0=0, i=0.$
		\WHILE {$F(S) < C_1  - \epsilon$}
		\STATE $j\in {\arg\max}_{e \in V\setminus S_{i-1}} F (e|S_{i-1})$
		\STATE $S_i = S_{i-1}\cup \{j\}$
		%\STATE $\sigma_i = j$
		\STATE $i = i + 1$
		\ENDWHILE
		\FOR {$j=1$ to $|S|$}
		\STATE{$\gamma_j = \sum_{i\in V} \mathbb{I} \big[ j = {\arg\min}_{s \in S} {\max_{w \in \mathcal{W}}}  \|\mathbf{H}^{-1}_t \textbf{g}_t-\sum_{j\in S}\gamma_{t,j}\mathbf{H}_{t,j}^{-1} \textbf{g}_{t,j}\|  \big]$} %\\
		\ENDFOR
	\end{algorithmic}
	\caption{\textsc{\alg} (Adaptive Coresets for Accelerating first and second order optimization methods) }
	\label{alg:greedy}
\end{algorithm}
	\vspace{-1mm}\textbf{}

\section{Bounding the Norm of Difference Between Preconditioned Gradients}

\subsection{Convex Loss Functions}\label{proof:boundnormerror}
We show the normed difference for ridge regression. Similar results can be deduced for other loss functions such as square loss \cite{allen2016exploiting}, logistic loss, smoothed hinge losses, etc.

For ridge regression $f_i(w)=\frac{1}{2} ( \langle x_i, w \rangle - y_i )^2 + \frac{\lambda}{2} \| w\|^2$, we have $\nabla f_i(w) = x_i (\langle x_i, w \rangle - y_i) + \lambda w$. Furthermore for invertable Hessian H, let $H^{-1}_i = A$ and $H^{-1}_j = B$.
Therefore,

 \begin{align}
 \|A \nabla f_i(w)-B \nabla f_j(w)\| =\| A (x_i\langle x_i,w\rangle - x_iy_i + \lambda w) - B(x_j \langle x_j,w \rangle - x_jy_j + \lambda w) \|
 \end{align}
 \begin{align}
&= \| A x_i\langle x_i,w\rangle - Bx_j \langle x_j,w\rangle + Bx_j y_j - Ax_iy_i + \lambda(A-B)w \|\\
&= \| A x_i \langle x_i,w\rangle + Bx_j \langle x_i,w\rangle -Bx_j \langle x_i,w\rangle - B x_j \langle x_j,w\rangle \nonumber \\
&\quad\quad\quad\quad\quad + Bx_jy_j - Ax_iy_i + Bx_jy_i -Bx_jy_i + \lambda(A+B)w\|\\
&= \| \langle x_i,w \rangle (A x_i - Bx_j) + \langle x_i - x_j,w \rangle Bx_j + (y_j-y_i)Bx_j + y_i(Bx_j-Ax_i) + \lambda (A-B)w \|\\
&= \| (\langle x_i,w \rangle - y_i)(A x_i - Bx_j) + (\langle x_i - x_j,w \rangle + y_j -y_i )B x_j + \lambda (A+B)w \|\\
&\leq | \langle x_i,w \rangle - y_i | \|A x_i - Bx_j \| + |\langle x_i - x_j,w \rangle + y_j -y_i | \| Bx_j\| +  \lambda \|(A-B)w\|\\
&\leq  | \langle x_i,w \rangle - y_i | ( \|A  - B \| \| x_i\| + \| B\| \| x_i - x_j\|) + |\langle x_i - x_j,w \rangle + y_j -y_i | \| Bx_j\| \nonumber \\
&\quad\quad\quad\quad\quad+  \lambda \|(A+B)w\| \\
&\leq O(\|w\|)( \| A - B \| + \| B\| \|x_i - x_j \| ) + O(\|w\|)\|B\| \|x_j\| \|x_i-x_j\| \\
&\leq O(\|w\|)( \| A\| +\| B \| + \| B\| \|x_i - x_j \| ) + O(\|w\|)\|B\| \|x_j\| \|x_i-x_j\| \label{eq:norm_inv_hess} 
\end{align}
% Line 63 to 64
In Eq. \eqref{eq:norm_inv_hess} we have the norm of the inverse of the Hessian matrix. Since H is invertible we have min$_i \sigma_i > 0,$
\begin{align}
    \min _{i} \sigma_{i}=\inf _{x \neq 0} \frac{\|H x\|_{2}}{\|x\|_{2}} \Longleftrightarrow \frac{1}{\min _{i} \sigma_{i}}=\sup _{x \neq 0} \frac{\|x\|_{2}}{\|H x\|_{2}}\\
    \frac{1}{\min _{i} \sigma_{i}}=\sup _{x \neq 0} \frac{\|x\|_{2}}{\|H x\|_{2}}=\sup _{H^{-1} z \neq 0} \frac{\left\|H^{-1} z\right\|_{2}}{\|z\|_{2}}=\sup _{z \neq 0} \frac{\left\|H^{-1} z\right\|_{2}}{\|z\|_{2}}=\left\|H^{-1}\right\|_{2},
\end{align}
where the substitution $Hx = z$ was made, and utilized that $H^{-1}z = 0 \Longleftrightarrow z = 0$ since H is invertible.
Hence,
\begin{align}
&\leq O(\|w\|) \|B\| \|x_i - x_j\| \\
&\leq O(\|w\|) \|x_i - x_j \|
 \end{align}
For $\| x_i \| \leq 1$, and %$y_i=y_j$ 
$|y_i-y_j|\approx0$ .   

Assuming that $\|w\|$ is bounded for all $w \in \mathcal{W}$, an upper bound on the euclidean distance between preconditioned gradients can be precomputed. 

\subsection{Neural Networks}\label{proof:boundnormederrornn}
We closely follow proofs from  \cite{katharopoulos2018not} and \cite{mirzasoleiman2020coresets} to show that we can bound the difference between the Hessian inverse preconditioned gradients of an entire NN up to a constant of the difference between the Hessian inverse preconditioned gradients of the last layer of the NN, between arbitrary datapoints $i$ and $j$.\\

 Consider an $L$-layer perceptron, where $w^{(l)}\in \mathbb{R}^{M_lxM_{l-1}}$ is the weight matrix for the $l^{th}$ layer with $M_l$ hidden units. Furthermore assume $\sigma^{(l)}(.)$ is a Lipschitz continuous activation function. Then we let, 

\begin{align}
x_{i}^{(0)} &=x_{i}, \\
z_{i}^{(l)} &=w^{(l)} x_{i}^{(l-1)}, \\
x_{i}^{(l)} &=\sigma^{(l)}\left(z_{i}^{(l)}\right) .
\end{align}
With,
\begin{align}
\Sigma_{l}^{\prime}\left(z_{i}^{(l)}\right) &=\operatorname{diag}\left(\sigma^{\prime(l)}\left(z_{i, 1}^{(l)}\right), \cdots \sigma^{\prime(l)}\left(z_{i, M_{l}}^{(l)}\right)\right) \\
\Delta_{i}^{(l)} &=\Sigma_{l}^{\prime}\left(z_{i}^{(l)}\right) w_{l+1}^{T} \cdots \Sigma_{l}^{\prime}\left(z_{i}^{(L-1)}\right) w_{L}^{T}.
\end{align}
We have,

\begin{align}
\| \mathbf{H}_{i}^{-1} \textbf{g}_{i}-& \mathbf{H}_{j}^{-1}\textbf{g}_{j} \| \\
=&\left\|\left(\Delta_{i}^{(l)} \Sigma_{L}^{\prime}(z_{i}^{(L)}) (\mathbf{H}_{i}^{-1})^{(L)}\textbf{g}_{i}^{(L)}\right)(x_{i}^{(l-1)})^{T}-\left(\Delta_{j}^{(l)} \Sigma_{L}^{\prime}(z_{j}^{(L)}) (\mathbf{H}_{j}^{-1})^{(L)}\textbf{g}_{j}^{(L)}\right)(x_{j}^{(l-1)})^{T}\right\| \\
\leq &\left\|\Delta_{i}^{(l)}\right\| \cdot\left\|x_{i}^{(l-1)}\right\| \cdot\left\|\Sigma_{L}^{\prime}\left(z_{i}^{(L)}\right) (\mathbf{H}_{i}^{-1})^{(L)}\textbf{g}_{i}^{(L)}-\Sigma_{L}^{\prime}\left(z_{j}^{(L)}\right) (\mathbf{H}_{j}^{-1})^{(L)}\textbf{g}_{j}^{(L)}\right\| \\
&+\left\|\Sigma_{L}^{\prime}\left(z_{j}^{(L)}\right) (\mathbf{H}_{i}^{-1})^{(L)}\textbf{g}_{i}^{(L)}\right\| \cdot\left\|\Delta_{i}^{(l)}\left(x_{i}^{(l-1)}\right)^{T}-\Delta_{j}^{(l)}\left(x_{j}^{(l-1)}\right)^{T}\right\| \nonumber \\
\leq &\left\|\Delta_{i}^{(l)}\right\| \cdot\left\|x_{i}^{(l-1)}\right\| \cdot\left\|\Sigma_{L}^{\prime}\left(z_{i}^{(L)}\right) (\mathbf{H}_{i}^{-1})^{(L)}\textbf{g}_{i}^{(L)}-\Sigma_{L}^{\prime}\left(z_{j}^{(L)}\right) (\mathbf{H}_{j}^{-1})^{(L)}\textbf{g}_{j}^{(L)}\right\| \\
&+\left\|\Sigma_{L}^{\prime}\left(z_{j}^{(L)}\right) (\mathbf{H}_{i}^{-1})^{(L)}\textbf{g}_{i}^{(L)}\right\| \cdot\left(\left\|\Delta_{i}^{(l)}\right\| \cdot\left\|x_{i}^{(l-1)}\right\|+\left\|\Delta_{j}^{(l)}\right\| \cdot\left\|x_{j}^{(l-1)}\right\|\right) \nonumber \\
\leq & \underbrace{\max _{l}\left(\left\|\Delta_{i}^{(l)}\right\| \cdot\left\|x_{i}^{(l-1)}\right\|\right)}_{c_{l, i}} \cdot\left\|\Sigma_{L}^{\prime}\left(z_{i}^{(L)}\right) (\mathbf{H}_{i}^{-1})^{(L)}\textbf{g}_{i}^{(L)}-\Sigma_{L}^{\prime}\left(z_{j}^{(L)}\right) (\mathbf{H}_{j}^{-1})^{(L)}\textbf{g}_{j}^{(L)}\right\| \\
&+\underbrace{\left\|\Sigma_{L}^{\prime}\left(z_{i}^{(L)}\right) (\mathbf{H}_{i}^{-1})^{(L)}\textbf{g}_{i}^{(L)}\right\| \cdot \max _{l, i, j}\left(\left\|\Delta_{i}^{(l)}\right\| \cdot\left\|x_{i}^{(l-1)}\right\|+\left\|\Delta_{j}^{(l)}\right\| \cdot\left\|x_{j}^{(l-1)}\right\|\right)}_{c_{2}} \nonumber
\end{align}
From \cite{katharopoulos2018not}, \cite{mirzasoleiman2020coresets}, we have that the variation of the gradient norm is mostly captured by the gradient of the loss function with respect to the pre-activation outputs of the last layer of our neural network. Here we have a similar result, where, the variation of the gradient preconditioned on the inverse of the Hessian norm is mostly captured by the gradient preconditioned on the inverse of the Hessian of the loss function with respect to the pre-activation outputs of the last layer of our neural network. Assuming $\left\|\Sigma_{L}^{\prime}\left(z_{i}^{(L)}\right) (\mathbf{H}_{i}^{-1})^{(L)}\textbf{g}_{i}^{(L)}\right\|$ is bounded, we get
\begin{align}
    \| \mathbf{H}_{i}^{-1} \textbf{g}_{i}-& \mathbf{H}_{j}^{-1}\textbf{g}_{j} \| \leq c_1  \left\|\Sigma_{L}^{\prime}\left(z_{i}^{(L)}\right) (\mathbf{H}_{i}^{-1})^{(L)}\textbf{g}_{i}^{(L)}-\Sigma_{L}^{\prime}\left(z_{j}^{(L)}\right) (\mathbf{H}_{j}^{-1})^{(L)}\textbf{g}_{j}^{(L)}\right\| + c_2
\end{align}
where $c_1, c_2$ are constants. The above holds for an affine operation followed by a slope-bounded non-linearity $\left(\left|\sigma^{\prime}(w)\right| \leq K\right)$.

\subsection{Analytic Hessian for Logistic Regression}\label{appx:hessian}
Here we provide the analytical formulation of the Hessian of the binary cross entropy loss per data point $n$ with respect to weights $\mathbf{w}$ for Logistic Regression. 

For Binary Logistic Regression we have a loss function $\Lc(\mathbf{w})$ defined as:
\begin{align}
    \Lc(\mathbf{w}) =- \sum_{i=1}^{N} l_i(\mathbf{w})=  - \sum_{i=1}^{N} {y_i ln(\hat{\sigma}) + (1-y_i) ln(1-\hat{\sigma})}, \text{ where } \hat{\sigma}_i = \sigma(\mathbf{w^Tx_i}+b)
\end{align}
and $\sigma$ is the sigmoid function. \\

We form a Hessian matrix for each data point $i$ based on loss function $\ell_i(\mathbf{w})$ as follows:

\[
H_n = \left(\begin{array}{@{}c|c@{}}
  
    \frac{\partial}{\partial \mathbf{w^2}} l_i(\mathbf{w})
  & \frac{\partial  }{\partial \mathbf{w} \partial b} l_i(\mathbf{w}) \\
\hline
  \frac{\partial  }{\partial b \partial \mathbf{w}} l_i(\mathbf{w}) &
  \begin{matrix}
\frac{\partial  }{\partial b \partial b} l_i(\mathbf{w})
  \end{matrix}
\end{array}\right) 
= 
\left(\begin{array}{@{}c|c@{}}
  \begin{matrix}
    \hat{\sigma}_i(1-\hat{\sigma}_i)\mathbf{ x_i x_i^T}
  \end{matrix}
  & \hat{\sigma}_i(1-\hat{\sigma}_i)\mathbf{ x_i }\\
\hline
  {[\hat{\sigma}_i(1-\hat{\sigma}_i) \mathbf{x_i}]^T} &
  \begin{matrix}
    \hat{\sigma}_i (1-\hat{\sigma}_i)
  \end{matrix}
\end{array}\right) 
\]
This allows us to analytically form the Hessian information per point which is needed to precompute a single coreset which will be used throughout training of the convex regularized logistic regression problem.

\section{Further Empirical Evidence}\label{apx:empirical}

\subsection{\alg estimates full gradient closely, reaching smaller loss}

 \alg obtains a better estimate of the preconditioned gradient by considering curvature and gradient information compared to the state-of-the-art algorithm \craig and random subsets. This is quantified by calculating the difference between weighted gradients of coresets and the gradient of the complete dataset.

Fig \ref{fig:normed_grad_diff_subsets}, shows the difference in loss reached by \alg vs \craig over different subset sizes. This shows that corsets selected using \alg to classify the Ijcnn1 dataset using logistic regression can reach a lower loss over varying subset sizes than \craig . 
% \clearpage

\begin{figure}[ht]
    \centering
% 	\vspace{-5mm}
	\includegraphics[width=.4\textwidth,trim=10mm 0 0 10mm 0]{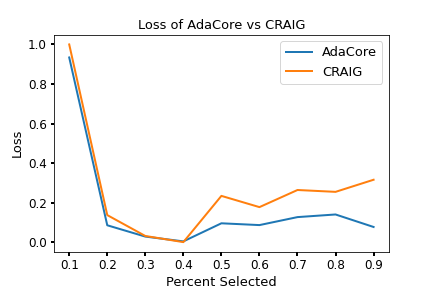}
    \caption{Normalized loss for Logistic Regression over different subset sizes on Ijcnn1 dataset using SGD. \alg corsets, considering curvature information, to classify Ijcnn1 dataset using logistic regression consistently reaches a lower loss compared to CRAIG, which only considers the gradient information.
    }
    \label{fig:normed_grad_diff_subsets}
\end{figure}

\subsection{Class imbalance CIFAR-10}
\label{exp:clasimn}
To provide further empirical evidence, we include results using a class-imbalanced version of the CIFAR-10 dataset for ResNet18. We skewed the class distribution linearly, keeping $90\%$ of class 9, $80\%$ of class 8 \dots $10\%$ of class 1, and $0\%$ of class 0, and trained for 200 epochs. Selecting a coreset for every epoch can be computationally expensive; instead, one can compute a coreset once every $R$ epochs. Here we investigate \alg's performance on various $R$ values. As Table \ref{table:class_imb_sgd_mom_full} shows, \alg can withstand class imbalance much better than \craig and randomly selected subsets. When $R=20$, \alg achieves $57.3\%$ final test accuracy, $+8.7\%$ above \craig, $+2.6\%$ above Random, 27.4\% above \gradmatch and 36.2\% above \glister. 

\begin{table}[h!]
\caption{CIFAR-10 Class Imbalance, ResNet18. Final test accuracy and percent of full data selected (in parentheses). Trained with SGD + Momentum, selecting a coreset every $R$ epochs that is $S$ percent of the full dataset. Note \alg has greater accuracy while seeing fewer data points.}
\label{table:class_imb_sgd_mom_full}
\begin{tabular}{llll}
\hline
Accuracy & {\color[HTML]{333333} $S = 1\% R = 20$} & $S = 1\% R = 10$ & $S = 1\% R = 5$ \\ \hline
\alg & $\textbf{57.3\%} \pm 0.5$ \; $(\textbf{5\%})$ & $\textbf{57.12} \pm 0.96$ \; $(\textbf{9.5\%})$ &$ \textbf{60.2\%} \pm 0.36$ \; $(\textbf{14.5\%})$ \\
\craig & $48.6\%  \pm 0.8$ \; $(8\%)$ & $55 \pm 1 $\;$(16\%)$ & $53.05\% \pm 0.24$ \;$(27.5\%)$ \\
Random & $54.7\% \pm 0.3 $\;$(8\%)$ & $54.6 \pm 0.76$ \;$(18\%)$ & $54.6\% \pm 0.74$ \;$(33.2\%)$\\
\gradmatch & $29.9\% \pm 0.4$ \;$(8.2\%)$ & $29.1\% \pm 0. 8$\; $(14.7\%)$ & $32.75\% \pm 0.83$ \;$ (23.2\%)$\\
\glister & $21.1\% \pm 0.42$ \;$(8.6\%)$ &$ 17.2\% \pm 0.75$\;$(16\%)$ & $14.4\% \pm 0.83$ \; $(22.2\%)$
\end{tabular}
\end{table}

Not only is \alg able to outperform \craig, random, \gradmatch, and \glister, but it can do so while selecting a smaller fraction of the data points during training, as shown under all settings in Table \ref{table:class_imb_sgd_mom_full}.

\subsection{Class imbalance BDD100k}\label{exp:bdd100k}
Additionally, we compared \alg to \craig and random selection for the BDD100k dataset, which has seven inherently imbalanced classes and 100k data points. We train ResNet50 with SGD + momentum for 100 epochs choosing subset size (s = 10\%) every (R = 20) epochs on the weather prediction task. We see that \alg can outperform \craig by 2\% and random by 8.8\% seen in Table \ref{table:bdd100kepoch}. 
\begin{table}[ht]\caption{\alg outperforms other baseline subset selection algorithms as well as training on the full dataset, reaching a better accuracy in less time. This provides up to a 2.3x speedup compared to to the state of the art.}
\begin{tabular}{ll}
\hline
\textbf{SGD + Momentum Accuracy} & {\color[HTML]{333333} \textbf{S = 10\% R = 20}} \\ \hline
\alg                          & 74.3\%                                          \\
\craig                            & 72.3\%                                          \\
Random                           & 65.5\%                                         
\end{tabular}\label{table:bdd100kepoch}
\end{table}\\

Additionally, Table \ref{table:bdd100k} shows that \alg outperforms baseline methods on BDD100k providing 2.3x speedup vs. training on the entire dataset and a 1.8x speedup vs. random. We see that \craig, \gradmatch \& \glister do not reach the accuracy of \alg even given more time and epochs. The epoch value is seen in parenthesis by accuracy. These experiments were run with SGD+momentum.

\begin{table}[ht]\caption{\alg outperforms other baseline subset selection algorithms as well as training on the full dataset, reaching a better accuracy in less time. This provides up to a 2.3x speedup compared to to the state of the art.}
\begin{tabular}{@{}lllll@{}}
\toprule
\textbf{}                         & \textbf{BDD100k} & \textbf{} & \multicolumn{2}{l}{Speedup over} \\ \midrule
\multicolumn{1}{l|}{\begin{tabular}[c]{@{}l@{}}$S = 10\%$\\ R = 20\end{tabular}} &
  \begin{tabular}[c]{@{}l@{}}Accuracy\\ (epoch)\end{tabular} &
  \multicolumn{1}{c}{\begin{tabular}[c]{@{}c@{}}Time\\ (s)\end{tabular}} &
  \multicolumn{1}{c}{Rand} &
  \multicolumn{1}{c}{Full} \\ \midrule
\multicolumn{1}{l|}{\alg}      & $74.3\%(100) $     & 7331      & \textbf{1.8}    & \textbf{2.3}   \\
\multicolumn{1}{l|}{\craig}        & $73.1\%(150)$      & 10996     & 1.3             & 1.6            \\
\multicolumn{1}{l|}{Random}       & $73.3\%(180)$      & 13050     & 1               & 1.2            \\
\multicolumn{1}{l|}{\gradmatch}    & $72\%(200)$        & 14040     & .7              & 1.1            \\
\multicolumn{1}{l|}{\glister}      & $73\%(200)$        & 12665     & 1.03            & 1.2            \\
\multicolumn{1}{l|}{Full Dataset} & $74.3\% (45)$      & 16093     & 0.8             & 1              \\ 
\bottomrule
\end{tabular}\label{table:bdd100k}
\end{table}

\subsection{CIFAR-100}\label{exp:cifar100}
Table \ref{table:cifar100} shows that \alg outperforms baseline methods on CIFAR100, providing 4x speedup vs. training on the entire dataset and a 3.8x speedup vs. Random. We see that \craig, \gradmatch and \glister do not reach the accuracy of \alg even given more time and epochs. The epoch value is seen in parenthesis by accuracy. These experiments were run with SGD+momentum.
\begin{table}[ht]\caption{\alg outperforms other baseline subset selection algorithms as well as training on the full dataset, reaching a better accuracy in less time. This provides up to a 4.3x speedup compared to to the state of the art.}\vspace{-2mm}
\begin{tabular}{@{}lllll@{}}
\toprule
\textbf{}                         & \textbf{CIFAR100} &      & \multicolumn{2}{l}{Speedup over} \\ \midrule
\multicolumn{1}{l|}{\begin{tabular}[c]{@{}l@{}}S = 10\%\\ R = 20\end{tabular}} &
  \begin{tabular}[c]{@{}l@{}}Accuracy\\ (epoch)\end{tabular} &
  \multicolumn{1}{c}{\begin{tabular}[c]{@{}c@{}}Time\\ (s)\end{tabular}} &
  \multicolumn{1}{c}{Rand} &
  \multicolumn{1}{c}{Full} \\ \midrule
\multicolumn{1}{l|}{\alg}      & 58.8\%(200)       & 341  & \textbf{4.3}    & \textbf{2.8}   \\
\multicolumn{1}{l|}{\craig}        & 57.3\%(250)       & 426  & 3.5             & 2.2            \\
\multicolumn{1}{l|}{Random}       & 58.1\%(864)       & 1470 & 1               & 0.65           \\
\multicolumn{1}{l|}{\gradmatch}    & 57\%(200)         & 980  & 1.5             & 0.97           \\
\multicolumn{1}{l|}{\glister}      & 56\%(300)         & 1110 & 1.3             & 0.86           \\
\multicolumn{1}{l|}{Full Dataset} & 59\% (40)         & 960  & 1.5             & 1              \\ \bottomrule
\end{tabular}\label{table:cifar100}\vspace{-5mm}
\end{table}

\subsection{When first order coresets fail, continued}
\label{apx:wcf}

By preconditioning with curvature information, \alg is able to magnify smaller gradient dimensions that would otherwise be ignored during coreset selection. Moreover, it allows \alg to include points with similar gradients but different curvature properties. Hence, \alg can select more diverse subsets compared to \craig as well as \gradmatch.
This allows \alg to outperform first order coreset methods in many regimes, such as when subset size is large (e.g. $\geq$10\%) and for larger batch size (e.g. $\geq$ 128).

\begin{figure}[ht]
    \centering
    \subfloat[\alg with gradient w.r.t the penultimate layer, training with SGD + Momentum]{
    \includegraphics[width=.41\textwidth,trim=10mm 0  10mm 0]{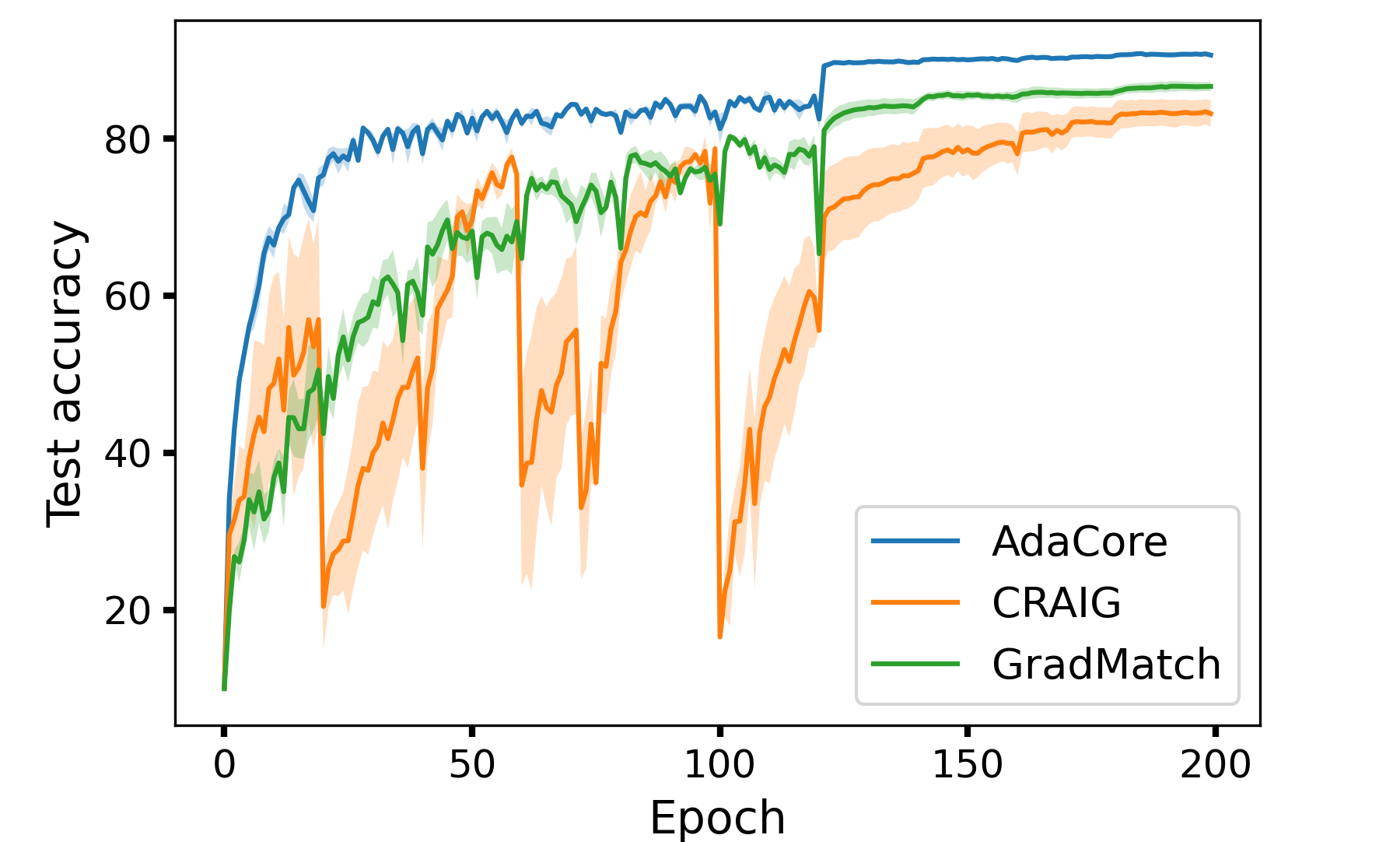}
	}
    \caption{Classification accuracy of ResNet20 across training on the CIFAR10 dataset, selecting coresets with \alg, \craig and \gradmatch. Here, all coreset selection methods used the gradients of the model's last layer (dimension 64). The algorithms were calculated every $R=20$ epochs with coreset size $S=10\%$. Note that \craig and \gradmatch are vulnerable to catastrophic forgetting, but not \alg.
    }
    \label{fig:when_craig_fails}
\end{figure}

In addition to the results shown in Figure \ref{subfig:R18_acc}, (reproduced here as Fig \ref{fig:repeat_craig_vs_ada}) where $R=1$, \alg outperforms \craig as well as \gradmatch when we increase the coreset selection period $R$. Fig \ref{fig:when_craig_fails} shows that for larger $R$, first-order methods succumb to catastrophic forgetting each time a new subset is chosen, whereas \alg achieves a smooth rise in classification accuracy. This increased stability between coresets is another benefit of \alg's greater selection diversity. 
Interestingly, \alg achieves higher final test accuracy while
selecting a smaller fraction of data points to train on during the training  than \craig. 
Note that since \alg takes curvature into account while selecting the coresets, it can successfully select data points with a similar gradient but different curvature properties and extract a more diverse set of data points than \craig. However, as the coresets found by \alg provide a close estimation of the full preconditioned gradients for several epochs during training, the number of distinct data points selected by \alg is smaller than \craig.

\begin{figure}[t]
    \centering
	\vspace{-2mm}
    \subfloat[Accuracy vs. Epoch \label{fig:repeat_craig_vs_ada}]{
	\includegraphics[width=.41\textwidth,trim=10mm 0 0 10mm 0]{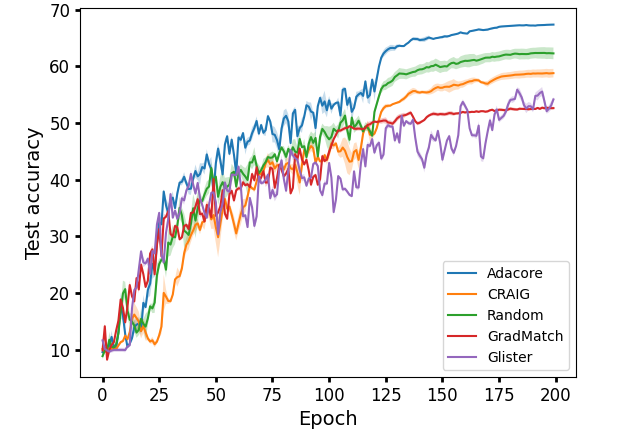}
    }~
    \subfloat[Accuracy vs. Time \label{fig:repeat_craig_vs_ada_1000}]{
    \includegraphics[width=.36\textwidth,trim=10mm 2mm 0 10mm 0]{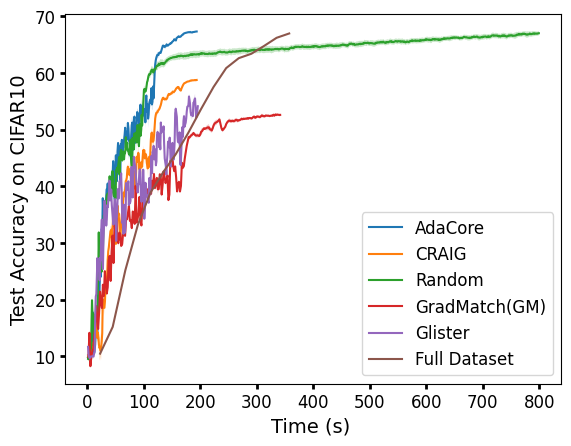}
    }
    \vspace{-2mm}
    \caption{ (a) Test accuracy of \alg, CRAIG, Random, GradMatch and GLISTER with ResNet-18 selecting subsets of size 1\% each epoch, batch size 256. (b) Training ResNet-18 on subsets of size $S$=1\% selected every $R$=1 epoch, with \alg, \craig, \glister and \gradmatch for 200 epochs vs. Random for 1000 epochs and full for 15 epochs. 
    \alg outperforms baselines by providing 2x speedup over full, and more than 4.5x speedup over Random. 
    }
    \vspace{-5mm}
    % \label{fig:speedup}
\end{figure}

For completeness we provide Fig \ref{fig:repeat_craig_vs_ada_1000}, in which we allow training random subset selection 1000 epochs. We see that it takes over 4.5x longer for Random to near the accuracy of ResNet18 trained with \alg and Full. We use the same experimental setup as seen in Fig \ref{fig:repeat_craig_vs_ada}.

\subsection{MNIST}
\label{exp:mnist}
For our MNIST classifier, we use a fully-connected hidden layer of 100 nodes and ten softmax output nodes; sigmoid activation and L2 regularization with $\mu = 10^{-4}$ and mini-batch size of 32 on the MNIST dataset of handwritten digits containing 60,000 training and 10,000 test images all normalized to [0,1] by division with 255. We apply SGD with a momentum of 0.9 to subsets of size 40\% of the dataset chosen at the beginning of each epoch found by \alg, CRAIG, and random.
Fig \ref{fig:mnist} compares the training loss and test accuracy of the network trained on coresets chosen by \alg, CRAIG, and random, with that of the entire dataset. We see that \alg can benefit from the second-order information and effectively finds subsets that achieve superior performance to that of baselines and the entire dataset. At the same time, it achieves a 2.5x speedup over training on the entire dataset.
\begin{figure}[ht]
    \centering
	\vspace{-4mm}
    \subfloat[MNIST ]{
	\includegraphics[width=.36\textwidth,trim=10mm 0 12mm 10mm]{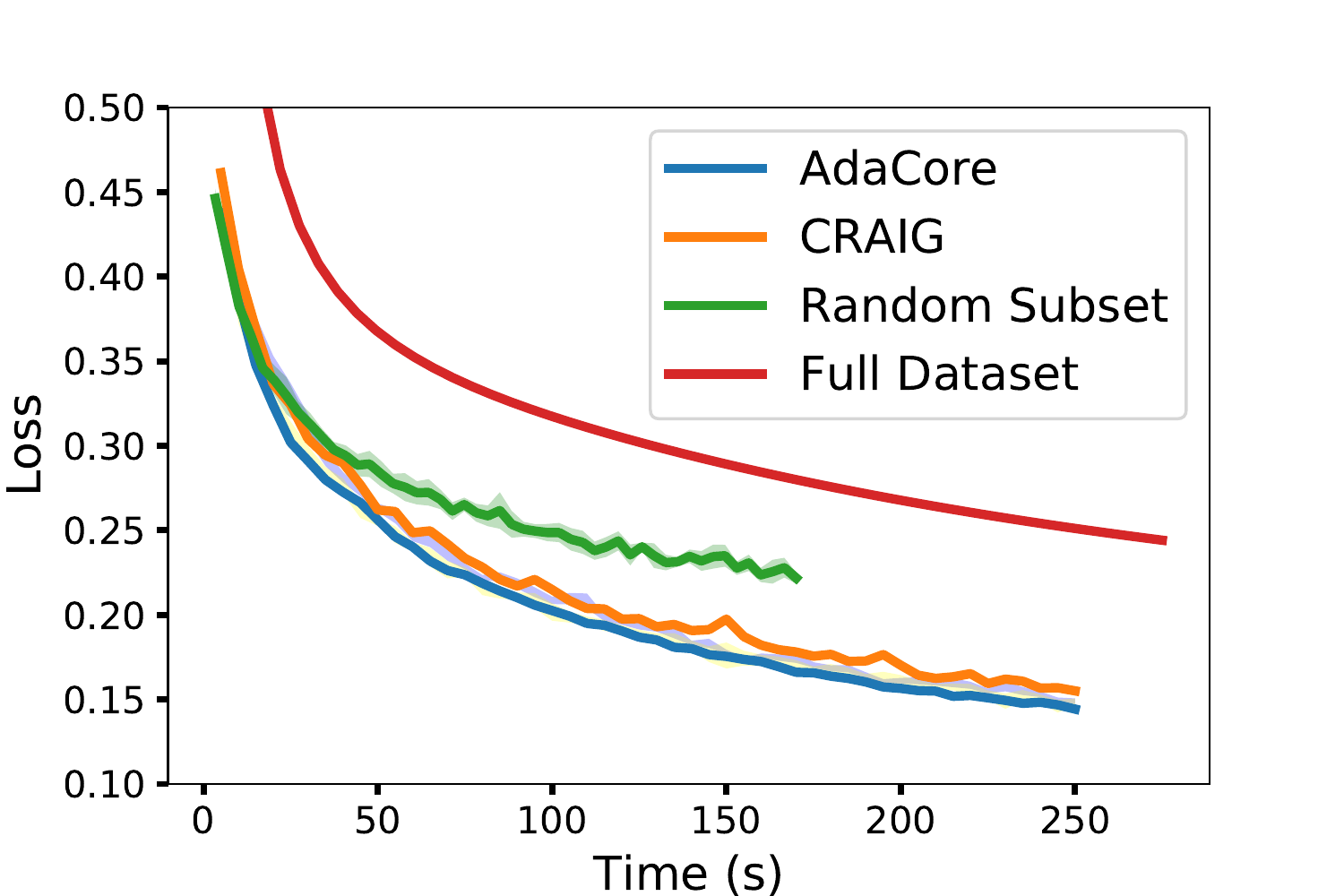}
    }\hspace{2mm}
    \subfloat[MNIST ]{
    \includegraphics[width=.36\textwidth,trim=10mm 0 12mm 10mm]{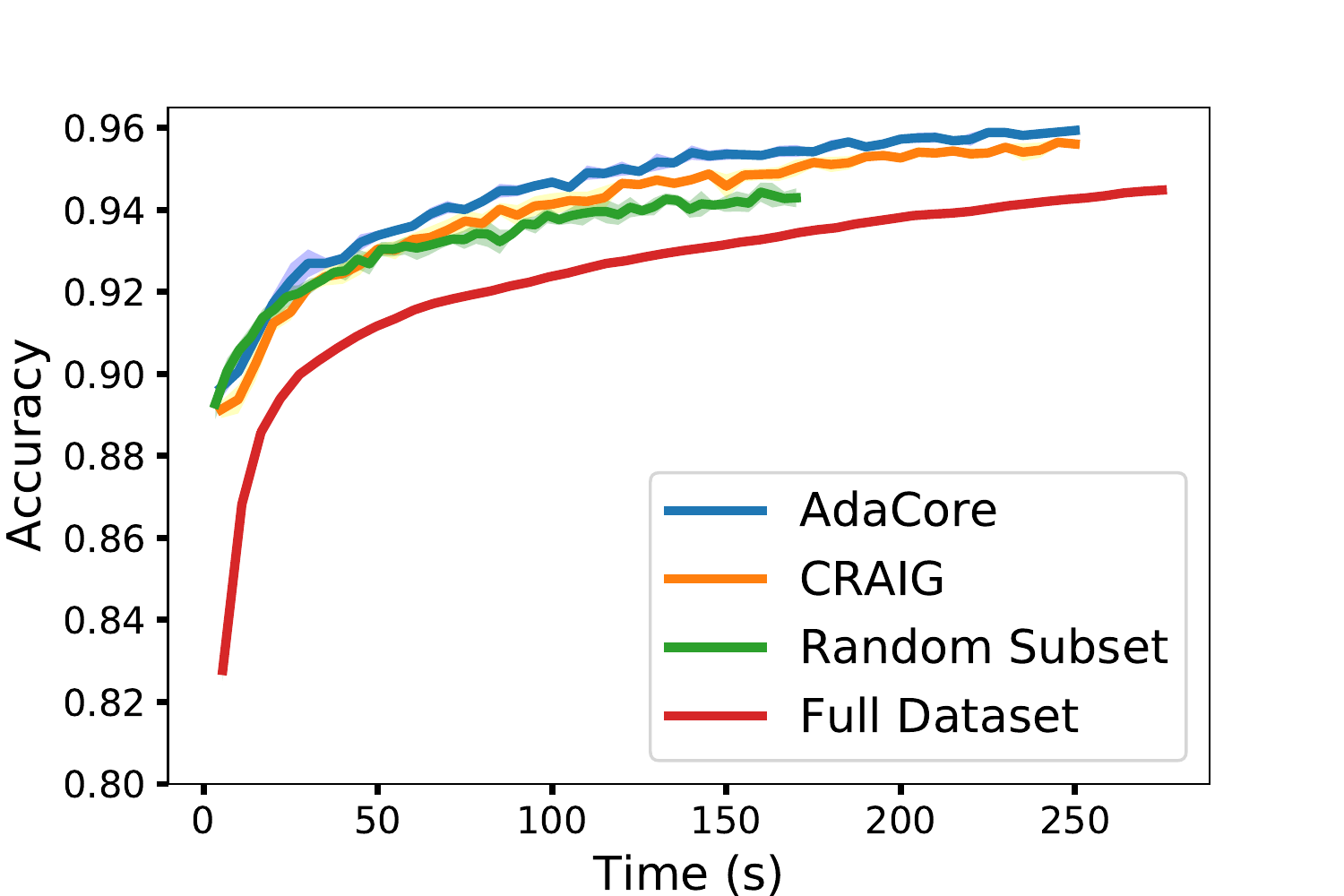}
    }
    \caption{Test accuracy and training loss of SGD with momentum applied to subsets found by \alg vs. CRAIG, and random subsets on MNIST with a 2-layer neural network. \alg achieves 2.5x speedup and better test accuracy, compared to training on full dataset.
    }
    % \vspace{-5mm}
    \label{fig:mnist}
\end{figure}

\subsection{How batch size affects coreset performance}
We see in Table \ref{table:batch_pm} that training with larger batch size on subsets selected by \alg can achieve a superior accuracy. We reproduce Table \ref{table:batch} here with standard deviation values. 
\begin{table}[ht]
% \vspace{-4mm}
\caption{Training ResNet18 with $S$=1\%  subsets every $R$=1 epoch from CIFAR10 using batch size $b$= 512, 256, 128. 
\alg can leverage larger mini-bath size and obtain a larger
accuracy gap to \craig and Random. For $b$=512, we have 1 mini-batch (GD). 
}\label{table:batch_pm}
\vspace{-3mm}
\begin{small}
\begin{tabular}{l|lllll}
\hline
 & \textsc{\alg} & \craig & Rand & \begin{tabular}[c]{@{}l@{}}Gap/\\ \craig\end{tabular} & \begin{tabular}[c]{@{}l@{}}Gap/\\ Rand\end{tabular} \\ \hline
\begin{tabular}[c]{@{}l@{}}GD~~ b=512\end{tabular}  & $58.32\% \pm 0.45$ & $56.32\% \pm 0.32$ & $49.14\% \pm 1.19$ & $1.69\%$ & $8.91\%$ \\
\begin{tabular}[c]{@{}l@{}}SGD b=256\end{tabular} & $68.23\% \pm 0.2$ & $58.3\% \pm 1.38$  & $60.7\% \pm 1.04$  & $9.93\%$ & $8.16\%$ \\
\begin{tabular}[c]{@{}l@{}}SGD b=128\end{tabular} & $66.89\% \pm 0.73$& $58.17\% \pm 1.34$ & $65.46\% \pm0.93$ & $8.81\%$ & $1.52\%$
\end{tabular}
\end{small}
\end{table}

\subsection{Potential Social Impacts}
Regarding social impact, our coreset method can outperform other methods in accuracy while selecting fewer data points over training and providing over 2.5x speedup. This will allow for a more efficient learning pipeline resulting in a lesser environmental impact. Our method can significantly decrease the financial and environmental costs of learning from big data. The financial costs are due to expensive computational resources, and environmental costs are due to the substantial energy consumption and the produced carbon footprint.

\end{document}